\documentclass[10pt,journal,compsoc]{IEEEtran}
\usepackage{cite}
\usepackage{amsmath}
\usepackage{amssymb}
\usepackage{graphicx}
\usepackage{multirow}
\usepackage{adjustbox}
\usepackage{array}
\usepackage{booktabs}
\usepackage{xcolor}
\usepackage{placeins}
\usepackage{eso-pic}
\usepackage{hyperref}
\hypersetup{pdftitle={Context-Aware Topological Distillation Graph Network for Asymmetric EEG\textendash{}Eye Movement Emotion Recognition}, pdfauthor={Peijun Shi, Xiujin Zhu, Zhenguo Lao, Chee-Onn Chow, Wei Ru Wong, Zhen Liu, Xiaobing Jiang}}

\newtheorem{proposition}{Proposition}


\newcommand{\biophoto}[1]{\IfFileExists{#1}%
  {\includegraphics[width=1in,height=1.25in,clip,keepaspectratio]{#1}}%
  {\fcolorbox{red}{white}{\parbox[c][\dimexpr1.25in-2\fboxsep-2\fboxrule\relax][c]{\dimexpr1in-2\fboxsep-2\fboxrule\relax}{\centering\scriptsize\color{red}photo}}}}

\newcommand{\pmr}[2]{#1{\scriptsize$\,\pm\,$#2}}
\newcommand{\sigup}{$^{\blacktriangle}$}
\newcommand{\sigdn}{$^{\triangledown}$}
\newcommand{\sg}{\mathrm{sg}}
\newcommand{\KL}{\mathrm{KL}}
\newcommand{\rsm}{\mathrm{softmax}_{\mathrm{row}}}
\newcommand{\Aphy}{\mathbf{A}^{\mathrm{phy}}}

\begin{document}
\AddToShipoutPictureFG*{\AtPageLowerLeft{\raisebox{0.3in}{\makebox[\paperwidth]{\parbox{\textwidth}{\centering\footnotesize This work has been submitted to the IEEE for possible publication. Copyright may be transferred without notice, after which this version may no longer be accessible.}}}}}

\title{Context-Aware Topological Distillation Graph Network for Asymmetric EEG--Eye Movement Emotion Recognition}

\author{Peijun~Shi$^{\ast}$, Xiujin~Zhu$^{\ast}$, Zhenguo~Lao, Chee-Onn~Chow,~\IEEEmembership{Senior~Member,~IEEE}, Wei~Ru~Wong, Zhen~Liu, and~Xiaobing~Jiang
\IEEEcompsocitemizethanks{\IEEEcompsocthanksitem $^{\ast}$P. Shi and X. Zhu contributed equally to this work.
\IEEEcompsocthanksitem P. Shi, Z. Lao, C.-O. Chow, and W. R. Wong are with the Department of Electrical Engineering, Faculty of Engineering, Universiti Malaya, 50603 Lembah Pantai, Kuala Lumpur, Malaysia.
\IEEEcompsocthanksitem X. Zhu, Z. Liu, and X. Jiang are with the Department of Neurosurgery, Union Hospital, Tongji Medical College, Huazhong University of Science and Technology, Wuhan, Hubei 430022, China, and with the Hubei Provincial Clinical Research Center for Brain Computer Interface, Huazhong University of Science and Technology, Wuhan, Hubei 430022, China.
\IEEEcompsocthanksitem Corresponding authors: Chee-Onn Chow (e-mail: cochow@um.edu.my) and Xiaobing Jiang (e-mail: 2004xh0835@hust.edu.cn).}}

\markboth{Submitted to IEEE Transactions on Affective Computing}%
{Shi \MakeLowercase{\textit{et al.}}: Context-Aware Topological Distillation Graph Network for EEG--Eye Movement Emotion Recognition}

\IEEEtitleabstractindextext{%
\begin{abstract}
EEG--eye movement emotion recognition can support mental-health monitoring and adaptive interfaces, but most fusion models treat the two signals as parallel features, classify each window alone and degrade when eye data are lost. We propose the Context-Aware Topological Distillation Graph Network (CTD-GNet), in which EEG features form the nodes of an electrode graph, and eye movements and preceding windows modulate its edges, nodes and readout, the edges through electrode-specific prompts. Bidirectional distillation of the graph, node relations and predictions couples this multimodal teacher with a single-window EEG student, improving the teacher and aligning the two graphs 35-fold. In leave-one-subject-out tests on SEED and SEED-IV against 17 baselines, CTD-GNet has the highest mean accuracy, 84.91\% and 76.17\%, significantly above 9 and 6 of 13 eye-based baselines after Holm correction, the smallest cross-subject spread among multimodal models and the highest worst-subject accuracy. It is the most accurate model at every rate of missing eye data on SEED-IV, keeps 66.1\% on SEED without eye data, more than any EEG-only model, and stays 9.3 and 11.9 points above all single-window fusion models under eye-movement noise of one standard deviation. On film clips unseen in training, it leads on SEED with 77.1\%.
\end{abstract}

\begin{IEEEkeywords}
Electroencephalography, emotion recognition, eye movements, graph neural networks, knowledge distillation, missing modalities, multimodal fusion.
\end{IEEEkeywords}}

\maketitle

\IEEEdisplaynontitleabstractindextext

\IEEEpeerreviewmaketitle

\IEEEraisesectionheading{\section{Introduction}\label{sec:introduction}}

\IEEEPARstart{E}{motion} recognition from physiological signals~\cite{alarcao2017emotions} is moving from the laboratory into mental-health monitoring, adaptive interfaces and driver monitoring. Electroencephalography (EEG) measures cortical activity, and eye tracking measures fixations, saccades, blinks and pupil size, which reflect visual attention and arousal~\cite{bradley2008pupil}. On SEED~\cite{zheng2015investigating} and SEED-IV~\cite{zheng2019emotionmeter}, combining the two signals outperforms either alone~\cite{zheng2019emotionmeter,lu2015combining,liu2022comparing}. Three problems keep this gain from reaching practice.

\begin{figure}[t]
\centering
\includegraphics[width=\columnwidth,page=1]{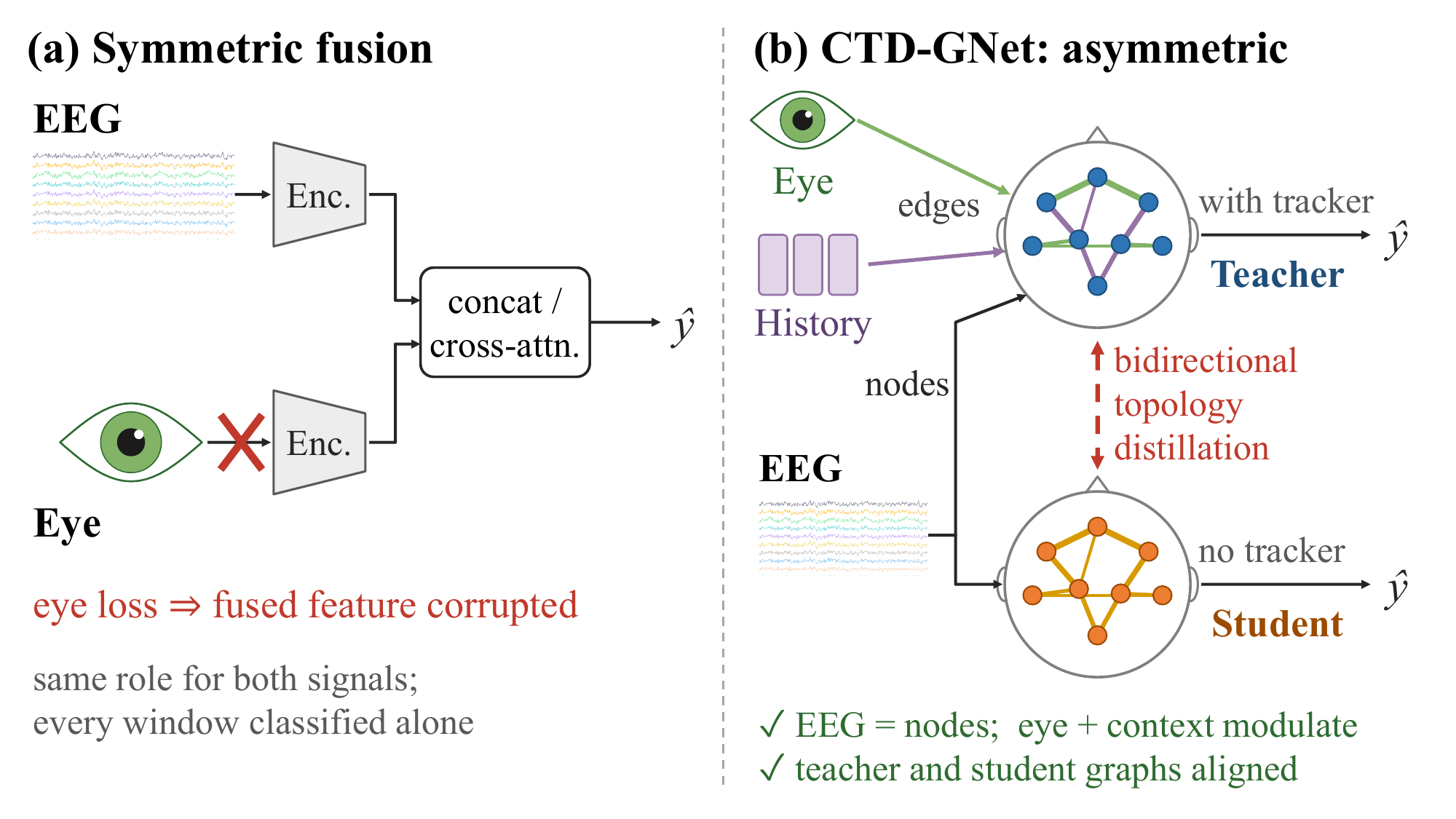}
\caption{Two ways to combine EEG and eye movements. (a)~Symmetric fusion encodes both signals as feature vectors and merges them by concatenation or cross-attention; missing eye data change the fused feature, and each window is classified independently. (b)~CTD-GNet: EEG gives the nodes of an electrode graph, and eye movements (green) and the preceding windows (purple) modulate its edges, nodes and readout. Bidirectional topology distillation couples the graph with that of an EEG-only student, which serves devices without an eye tracker.}
\label{fig:concept}
\end{figure}

First, fusion models treat the two signals alike. Concatenation, bimodal autoencoders~\cite{liu2016emotion}, canonical correlation~\cite{liu2022comparing}, tensor and low-rank fusion~\cite{zadeh2017tensor,liu2018efficient} and cross-modal attention~\cite{tsai2019multimodal} encode each signal as a vector and merge the vectors. Yet eye movements indicate where attention is directed, and attention changes which cortical regions exchange information~\cite{corbetta2002control}: eye movements describe how regions interact, whereas EEG describes what each region does. Merging the two as parallel vectors discards this difference.

Second, emotions induced by a film clip unfold over tens of seconds~\cite{kuppens2017emotion}, but most models classify each window alone, smooth the features over the trial~\cite{zheng2015investigating,shi2010offline} or feed the window sequence to a recurrent classifier~\cite{tang2017multimodal}. In all three cases, the preceding windows do not change how the regions of the current window are related.

Third, eye trackers lose data during blinks, head movements and calibration drift. Fusion models use both signals in every prediction, and many collapse without the eye data: in our experiments, cross-attention fusion falls to 24.0\% on SEED-IV when the eye data are removed, below chance. Missing-modality training~\cite{neverova2016moddrop} and reconstruction~\cite{zhao2021missing} let a fused model tolerate absent inputs, but the eye data still enter the fused representation directly, and multi-window models need a buffer of past windows on the device.

The three problems share one cause: eye movements and temporal context are treated as additional features rather than as signals that shape cortical interaction. We therefore build an electrode graph from EEG and let the eye movements and the context modulate its edges, nodes and readout. This design addresses the three problems together. The two signals take different roles; the preceding windows change the graph of the current window; and because the EEG evidence enters every graph computation, missing or corrupted eye data weaken the prediction instead of corrupting it. The effect of eye movements and context on connectivity is concentrated in the adjacency matrix, which a model that sees only EEG can learn to approximate.

We propose the Context-Aware Topological Distillation Graph Network (CTD-GNet), shown in Fig.~\ref{fig:concept}. Its teacher computes the edges by an attention over the 62 electrodes, with a recurrent context of the preceding windows in the queries and the eye movements in the keys; a gated node pathway and a gated multimodal unit~\cite{arevalo2017gated} add the eye movements to the nodes and the readout. Because eye movements and past windows are available during training but not always at deployment, we treat them as privileged information~\cite{vapnik2009new,lopezpaz2016unifying} and train an EEG-only student that sees only the current window. The two networks are trained jointly by bidirectional distillation of the graph, the node relations and the predictions. The main contributions of this work are as follows.
\begin{itemize}
\item We propose an asymmetric fusion paradigm in which EEG forms the nodes of an electrode graph and eye movements and temporal context modulate its edges, nodes and readout, so that the EEG evidence enters every computation; the resulting model has the highest worst-subject accuracy of all models, 74.02\% on SEED.
\item We introduce bidirectional distillation between the multimodal teacher and a single-window EEG student at the levels of the graph, the node relations and the predictions, which aligns their graphs 35-fold, so that both networks can be read on one connectivity, and also improves the teacher.
\item We design and evaluate the framework for unreliable sensors, with missing and noisy eye data and film clips unseen in training; with 80\% of the eye data missing on SEED-IV, the teacher is 4.1 points more accurate than the strongest baseline.
\end{itemize}

The code will be released at \url{https://github.com/Peijun1122/CTD-GNet}. The remainder of this paper is organized as follows. Section~\ref{sec:related_work} reviews related work. Section~\ref{sec:method} presents CTD-GNet. Section~\ref{sec:experiments} describes the experimental setup and results. Section~\ref{sec:discussion} discusses the findings, limitations and future work, and Section~\ref{sec:conclusion} concludes the paper.

\section{Related Work}
\label{sec:related_work}

\subsection{EEG-based emotion recognition}
Differential entropy (DE) features~\cite{duan2013differential} are standard on SEED and SEED-IV, and Zheng and Lu~\cite{zheng2015investigating} identified the $\beta$ and $\gamma$ bands and the lateral temporal electrodes as the most discriminative. Convolutional models~\cite{lawhern2018eegnet} learn from raw signals or feature maps, and hemispheric models~\cite{li2021bidann,li2021bihdm} use the difference between the hemispheres. Cross-subject models align the features of different subjects, for example by adversarial training~\cite{li2021bidann} or with class prototypes~\cite{zhou2024prpl}. Graph models treat electrodes as nodes: DGCNN~\cite{song2018eeg} learns one adjacency matrix, RGNN~\cite{zhong2022eeg} initializes it from electrode distances and hemispheric pairs, and instance-adaptive and self-organized graphs~\cite{song2020instance,li2021sognn} compute it from the EEG of each sample, as graph attention~\cite{velickovic2018graph} computes edge weights from node features. In all of them, the graph depends on EEG only.

\subsection{Multimodal fusion and temporal context}
EEG and eye movements have been fused at feature and decision level~\cite{lu2015combining}, by a bimodal deep autoencoder (BDAE)~\cite{liu2016emotion}, a bimodal LSTM~\cite{tang2017multimodal}, and deep canonical correlation analysis (DCCA)~\cite{andrew2013deep} with its attention variant DCCA-AM~\cite{liu2022comparing}. Recent models let one modality guide the encoding of the other~\cite{fu2024cross}, align the two modalities across subjects~\cite{zhu2024eeg,zhu2025multimodal,chen2025comprehensive} or adapt the fused features to the unlabeled data of a new subject~\cite{jimenez2024cfda}. General multimodal models use tensor products~\cite{zadeh2017tensor}, low-rank factors~\cite{liu2018efficient}, invariant and specific subspaces~\cite{hazarika2020misa}, gated units~\cite{arevalo2017gated} or cross-modal attention~\cite{tsai2019multimodal}. Most of these models fuse or align feature vectors of the two signals, so the eye movements do not shape the relations between electrodes, and a corrupted eye input alters the fused representation directly. Temporal context enters through smoothing of the features~\cite{shi2010offline,zheng2015investigating} or recurrent~\cite{cho2014learning} and temporal convolutional~\cite{bai2018empirical} networks over the window sequence; it changes the features or the output, but not the graph.

\subsection{Distillation, privileged information and missing modalities}
Distillation transfers outputs~\cite{hinton2015distilling}, features~\cite{romero2015fitnets}, sample relations~\cite{park2019relational,tung2019similarity} or the local structure of graph embeddings~\cite{yang2020distilling}, and deep mutual learning~\cite{zhang2018deep} distills between peers in both directions; across modalities, a visual teacher has been distilled into an EEG student for continuous emotion recognition~\cite{zhang2022visual}. Learning using privileged information~\cite{vapnik2009new} and generalized distillation~\cite{lopezpaz2016unifying} exploit data available only in training. Missing-modality methods drop modalities during training~\cite{neverova2016moddrop}, reconstruct them~\cite{zhao2021missing,ma2021smil}, model their shared and specific features~\cite{wang2023shaspec} or learn a prompt for each pattern of missing inputs~\cite{lee2023multimodal}, so that one fused model tolerates absent inputs; when the eye data are present, they still enter the fused features directly. Prompts~\cite{jia2022visual} and FiLM~\cite{perez2018film} condition networks with learned or input-dependent vectors, an idea that CTD-GNet applies to the edges of a graph. None of these methods transfers the graph through which a privileged modality acts, or couples the teacher to that graph in the opposite direction.

\section{Method}
\label{sec:method}

\begin{figure*}[t]
\centering
\includegraphics[width=\textwidth,page=2]{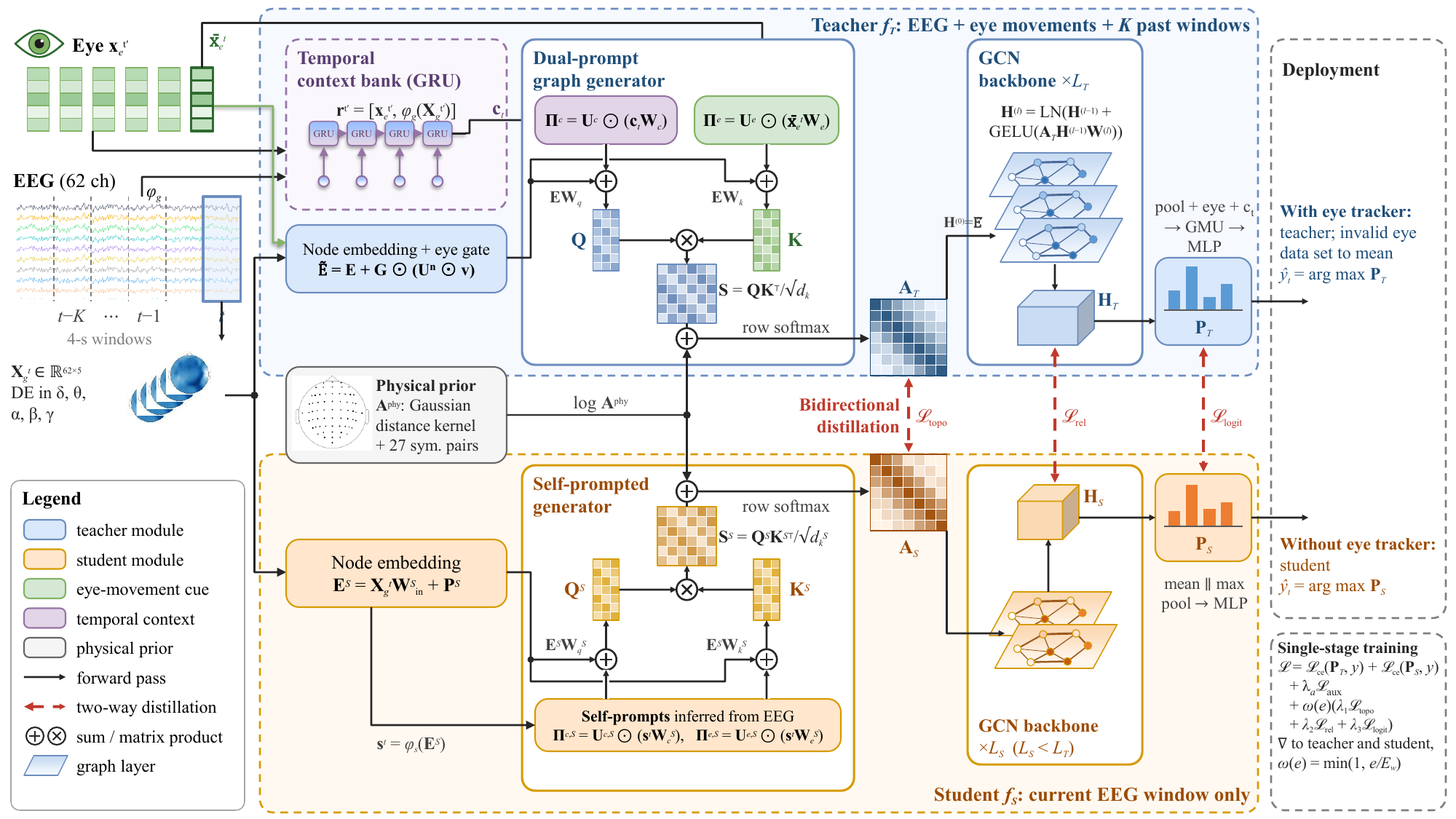}
\caption{Architecture of CTD-GNet. Teacher (blue): a GRU context bank encodes the $K$ preceding windows into $\mathbf{c}_t$; a gated node pathway adds the eye movements to the EEG node embeddings; the generator adds the context prompt $\boldsymbol{\Pi}^{c}$ to the queries, the eye prompt $\boldsymbol{\Pi}^{e}$ to the keys and $\log\Aphy$ to the scores, and a row-wise softmax gives $\mathbf{A}_T$; a GCN propagates over $\mathbf{A}_T$, and a gated multimodal unit fuses the graph readout with the eye movements and $\mathbf{c}_t$. Student (orange): the same operations with prompts from the current EEG window. Red dashed arrows: distillation losses, with gradients to both networks. Right: deployment and training objective.}
\label{fig:framework}
\end{figure*}

Fig.~\ref{fig:framework} shows the architecture. The teacher uses EEG, eye movements and context, the student uses the current EEG window, and both use the same physical prior. Section~\ref{sec:problem} defines the task, Section~\ref{sec:teacher} describes the teacher, Section~\ref{sec:student} the student, and Section~\ref{sec:distill} the distillation and deployment.

\subsection{Problem formulation}
\label{sec:problem}

Each trial is divided into non-overlapping windows indexed by $t$. Window $t$ provides the DE features $\mathbf{X}_g^t\in\mathbb{R}^{N\times d_g}$ of $N=62$ electrodes in $d_g=5$ frequency bands, the eye-movement features $\mathbf{x}_e^t\in\mathbb{R}^{d_e}$, and the label $y^t\in\{1,\dots,C\}$ of the trial. The history $\mathcal{H}_t=\{(\mathbf{X}_g^{t'},\mathbf{x}_e^{t'})\}_{t'=t-K}^{t-1}$ contains the $K$ preceding windows of the same trial. The teacher uses the eye cue $\bar{\mathbf{x}}_e^{t}=\frac{1}{K+1}\sum_{t'=t-K}^{t}\mathbf{x}_e^{t'}$, which averages tracker noise. The teacher $f_T$ maps $(\mathbf{X}_g^t,\mathbf{x}_e^t,\mathcal{H}_t)$ to logits $\mathbf{P}_T\in\mathbb{R}^{C}$, and the student $f_S$ maps $\mathbf{X}_g^t$ to $\mathbf{P}_S\in\mathbb{R}^{C}$; the eye movements and the history are privileged information~\cite{vapnik2009new} available only to the teacher. Vectors are row vectors, $\mathbf{M}_i$ is the $i$-th row of $\mathbf{M}$, $\rsm$ is the row-wise softmax, and $\sg(\cdot)$ is the stop-gradient operator.

\subsection{Context-aware teacher}
\label{sec:teacher}

The teacher consists of a temporal context bank, a dual-prompt graph generator and a graph backbone with a gated readout. The temporal context bank (TCB) summarizes the preceding windows with a gated recurrent unit (GRU)~\cite{cho2014learning}. Each history window $t'$ is mapped to a token
\begin{gather}
\mathbf{r}^{t'}=\big[\mathbf{x}_e^{t'},\ \phi_g(\mathbf{X}_g^{t'})\big]\in\mathbb{R}^{d_e+d_s},\nonumber\\
\phi_g(\mathbf{X})=\mathrm{vec}(\mathbf{X})\,\mathbf{W}_g+\mathbf{b}_g,
\label{eq:token}
\end{gather}
where $\phi_g$ linearly projects the 310-dimensional DE vector to $d_s$ dimensions. The GRU reads the sequence $\mathbf{r}^{t-K},\dots,\mathbf{r}^{t-1}$ in temporal order,
\begin{gather}
\boldsymbol{\xi}_k=\mathrm{GRU}\big(\mathbf{r}^{t-K-1+k},\,\boldsymbol{\xi}_{k-1}\big),\qquad k=1,\dots,K,\nonumber\\
\mathbf{c}_t=\mathrm{LN}(\boldsymbol{\xi}_K)\in\mathbb{R}^{d_c},
\label{eq:tcn}
\end{gather}
with $\boldsymbol{\xi}_0=\mathbf{0}$ and layer normalization LN. The GRU gates decide how much of the accumulated state each window keeps, so $\mathbf{c}_t$ follows the slow course of the induced emotion. The validation protocol selects between the GRU and a temporal convolutional network~\cite{bai2018empirical}. Windows at the start of a trial are padded with the first window, and histories do not cross trials.

The dual-prompt graph generator computes one adjacency matrix per window from node embeddings, two prompts and a physical prior. Each electrode $i$ is a node with the EEG embedding
\begin{equation}
\mathbf{e}_i=\mathbf{x}_{g,i}^{t}\mathbf{W}_{\mathrm{in}}+\mathbf{p}_i\in\mathbb{R}^{d_h},\qquad i=1,\dots,N,
\label{eq:embed}
\end{equation}
where $\mathbf{x}_{g,i}^{t}$ is the $i$-th row of $\mathbf{X}_g^t$ and $\mathbf{p}_i$ is a learnable electrode embedding; $\mathbf{E}\in\mathbb{R}^{N\times d_h}$ stacks the embeddings. The eye movements reach the nodes through a gate that each electrode sets for itself:
\begin{gather}
\mathbf{v}_t=\mathrm{LN}(\bar{\mathbf{x}}_e^{t}\mathbf{W}_n),\nonumber\\
\mathbf{g}_i=\sigma\big(\mathrm{LN}(\mathbf{e}_i)\mathbf{W}_{ge}+\mathbf{v}_t\mathbf{W}_{gv}+\mathbf{b}_n\big),\nonumber\\
\tilde{\mathbf{e}}_i=\mathbf{e}_i+\mathbf{g}_i\odot\mathbf{u}^{n}_i\odot\mathbf{v}_t,
\label{eq:nodegate}
\end{gather}
where $\sigma$ is the logistic function and $\mathbf{u}^{n}_i\in\mathbb{R}^{d_h}$ is an electrode embedding. Each electrode admits or blocks the eye input according to its own EEG state, and the EEG term $\mathbf{e}_i$ is always kept; $\tilde{\mathbf{E}}$ stacks the gated embeddings.

The edge weights come from an attention over the electrodes. The context modulates the queries, which determine what each electrode attends to, and the eye movements modulate the keys, which determine which electrodes act as sources:
\begin{gather}
\boldsymbol{\pi}^{c}_i=\mathbf{u}^{c}_i\odot(\mathbf{c}_t\mathbf{W}_c),\qquad
\boldsymbol{\pi}^{e}_j=\mathbf{u}^{e}_j\odot(\bar{\mathbf{x}}_e^{t}\mathbf{W}_e),
\label{eq:prompt}\\
\mathbf{q}_i=\tilde{\mathbf{e}}_i\mathbf{W}_q+\boldsymbol{\pi}^{c}_i,\qquad
\mathbf{k}_j=\tilde{\mathbf{e}}_j\mathbf{W}_k+\boldsymbol{\pi}^{e}_j,
\label{eq:qk}\\
S_{ij}=\mathbf{q}_i\mathbf{k}_j^{\top}/\sqrt{d_k},
\label{eq:score}
\end{gather}
where $\mathbf{W}_c\in\mathbb{R}^{d_c\times d_k}$ and $\mathbf{W}_e\in\mathbb{R}^{d_e\times d_k}$ are projections and $\mathbf{u}^{c}_i,\mathbf{u}^{e}_j\in\mathbb{R}^{d_k}$ are electrode prompt embeddings. The product with $\mathbf{u}$ gives each electrode its own version of a global cue, as FiLM~\cite{perez2018film} does for feature channels, and the matrices $\boldsymbol{\Pi}^{c}$ and $\boldsymbol{\Pi}^{e}$ stack the prompts.

\begin{figure*}[t]
\centering
\includegraphics[width=\textwidth,page=3]{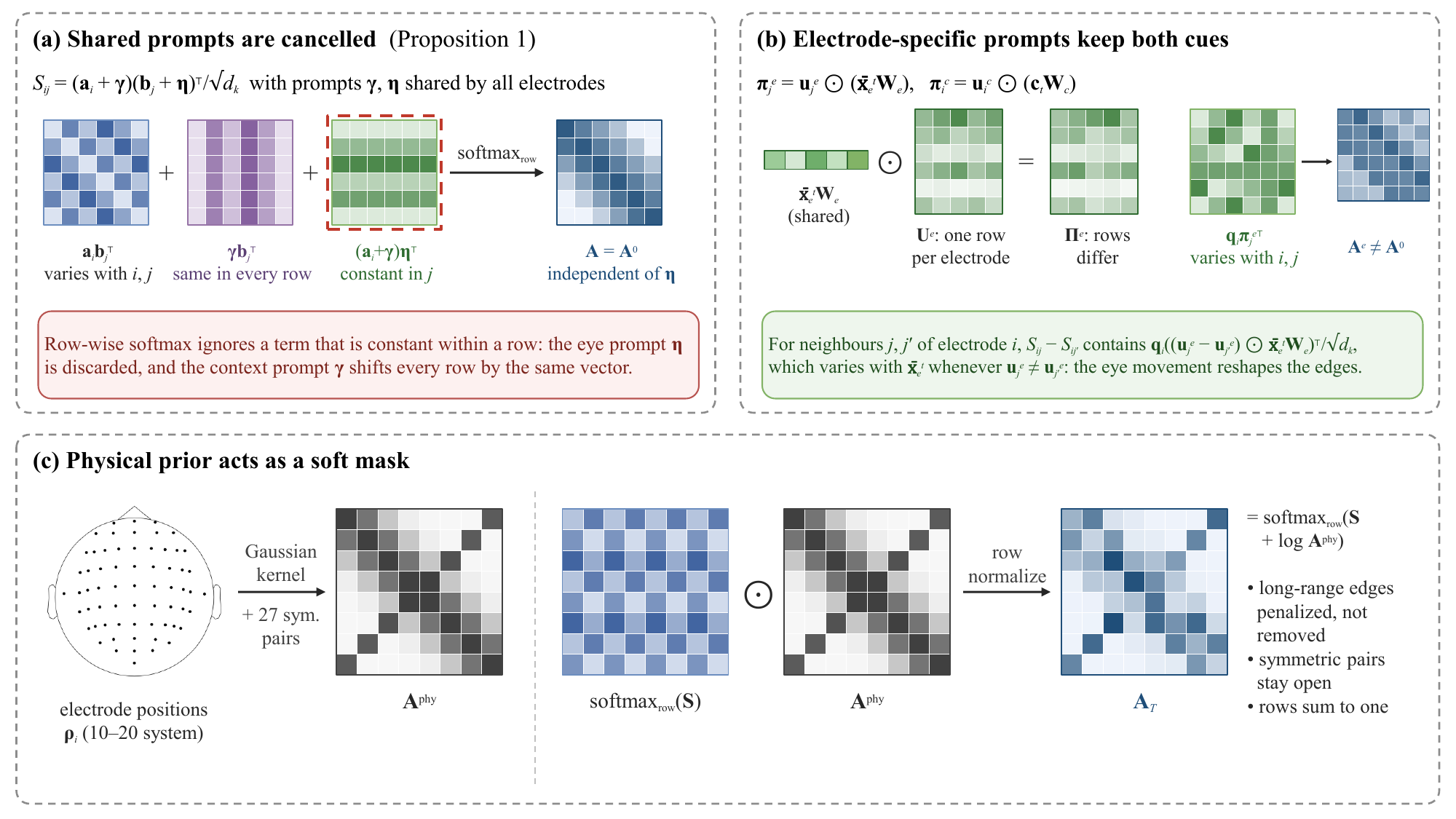}
\caption{Dual-prompt graph generator. (a)~Proposition~\ref{prop:cancel}: with shared prompts $\boldsymbol{\gamma}$ and $\boldsymbol{\eta}$, the term $(\mathbf{a}_i+\boldsymbol{\gamma})\boldsymbol{\eta}^{\top}$ is constant within each row (red frame) and is removed by the row-wise softmax. (b)~Electrode-specific prompts make $\mathbf{q}_i\boldsymbol{\pi}^{e\top}_j$ vary with $i$ and $j$, so the eye movements change the graph. (c)~The physical prior acts as a soft mask: $\rsm(\mathbf{S})\odot\Aphy$ normalized per row equals $\rsm(\mathbf{S}+\log\Aphy)$.}
\label{fig:prompt}
\end{figure*}

The prompts are electrode-specific by design. A simpler design adds one global embedding of each cue to all queries or keys; Proposition~\ref{prop:cancel}, which motivates our choice, shows that such a cue cannot change the graph electrode by electrode, and Fig.~\ref{fig:prompt}(a) illustrates the decomposition used in the proof.

\begin{proposition}[Prompt cancellation]
\label{prop:cancel}
Consider an attention matrix $\mathbf{A}=\rsm(\mathbf{S})$ with scores $S_{ij}=(\mathbf{a}_i+\boldsymbol{\gamma})(\mathbf{b}_j+\boldsymbol{\eta})^{\top}/\sqrt{d_k}$, where $\mathbf{a}_i$ and $\mathbf{b}_j$ depend on the electrode features and $\boldsymbol{\gamma},\boldsymbol{\eta}\in\mathbb{R}^{d_k}$ are prompts shared by all electrodes. Then (i) $\mathbf{A}$ does not depend on $\boldsymbol{\eta}$, and (ii) $\mathbf{A}$ depends on $\boldsymbol{\gamma}$ only through $\boldsymbol{\gamma}\mathbf{b}_j^{\top}/\sqrt{d_k}$, which adds the same vector to every row of $\mathbf{S}$.
\end{proposition}
\begin{IEEEproof}
Expanding the product gives $\sqrt{d_k}\,S_{ij}=\mathbf{a}_i\mathbf{b}_j^{\top}+\boldsymbol{\gamma}\mathbf{b}_j^{\top}+(\mathbf{a}_i+\boldsymbol{\gamma})\boldsymbol{\eta}^{\top}$. The last term does not depend on $j$. The softmax of a row does not change when a constant is added to all its entries, so this term cancels, which proves (i). The remaining prompt term $\boldsymbol{\gamma}\mathbf{b}_j^{\top}$ does not depend on $i$, which proves (ii).
\end{IEEEproof}

A global eye embedding added to the keys thus has no effect on $\mathbf{A}$, and a global context embedding added to the queries shifts all rows by the same vector, so that every electrode re-weights its neighbors in the same way. Proposition~\ref{prop:cancel} concerns additive cues, the usual way to condition attention with prompts~\cite{jia2022visual}; electrode-specific prompts keep this additive form and let a cue act on each electrode separately. With Eq.~\eqref{eq:prompt}, the logit difference between two neighbors $j$ and $j'$ of electrode $i$ contains $\mathbf{q}_i\big((\mathbf{u}^{e}_j-\mathbf{u}^{e}_{j'})\odot\bar{\mathbf{x}}_e^t\mathbf{W}_e\big)^{\top}/\sqrt{d_k}$, which depends on $\bar{\mathbf{x}}_e^t$ whenever $\mathbf{u}^{e}_j\neq\mathbf{u}^{e}_{j'}$, as Fig.~\ref{fig:prompt}(b) illustrates; the term $\boldsymbol{\pi}^{c}_i\mathbf{k}_j^{\top}$ likewise depends on the context and on both electrodes. Proposition~\ref{prop:cancel} is a design motivation rather than a predictor of accuracy: because the eye movements reach the model through several redundant routes, removing one of them has a limited effect in the ablations of Section~\ref{sec:ablation}.

Finally, a physical prior constrains the graph. Coupling between electrodes decreases with distance, and emotion processing involves homologous regions of the two hemispheres~\cite{davidson1992anterior,li2021bihdm}. The prior
\begin{equation}
A^{\mathrm{phy}}_{ij}=
\begin{cases}
1, & i=j\ \text{or}\ (i,j)\in\mathcal{P},\\[2pt]
\exp\!\big(-\|\boldsymbol{\rho}_i-\boldsymbol{\rho}_j\|_2^2/(2\sigma_{\rho}^2)\big), & \text{otherwise},
\end{cases}
\label{eq:phy}
\end{equation}
uses the 3-D electrode positions $\boldsymbol{\rho}_i$ of the extended 10--20 system, the set $\mathcal{P}$ of 27 symmetric pairs, such as FP1--FP2, and the median pairwise distance $\sigma_{\rho}$. Unlike RGNN~\cite{zhong2022eeg}, which uses a similar prior only for initialization, CTD-GNet applies it to every generated graph:
\begin{equation}
\mathbf{Z}_T=\mathbf{S}+\log\Aphy,\qquad \mathbf{A}_T=\rsm(\mathbf{Z}_T).
\label{eq:at}
\end{equation}
Since $A_{T,ij}=\exp(S_{ij})A^{\mathrm{phy}}_{ij}/\sum_{k}\exp(S_{ik})A^{\mathrm{phy}}_{ik}$, the prior acts as a soft mask, as Fig.~\ref{fig:prompt}(c) shows: distant pairs are penalized in proportion to their squared distance, symmetric pairs are not, and each row of $\mathbf{A}_T$ sums to one. We call $\mathbf{Z}_T$ the topology logits.

\begin{figure*}[t]
\centering
\includegraphics[width=\textwidth,page=4]{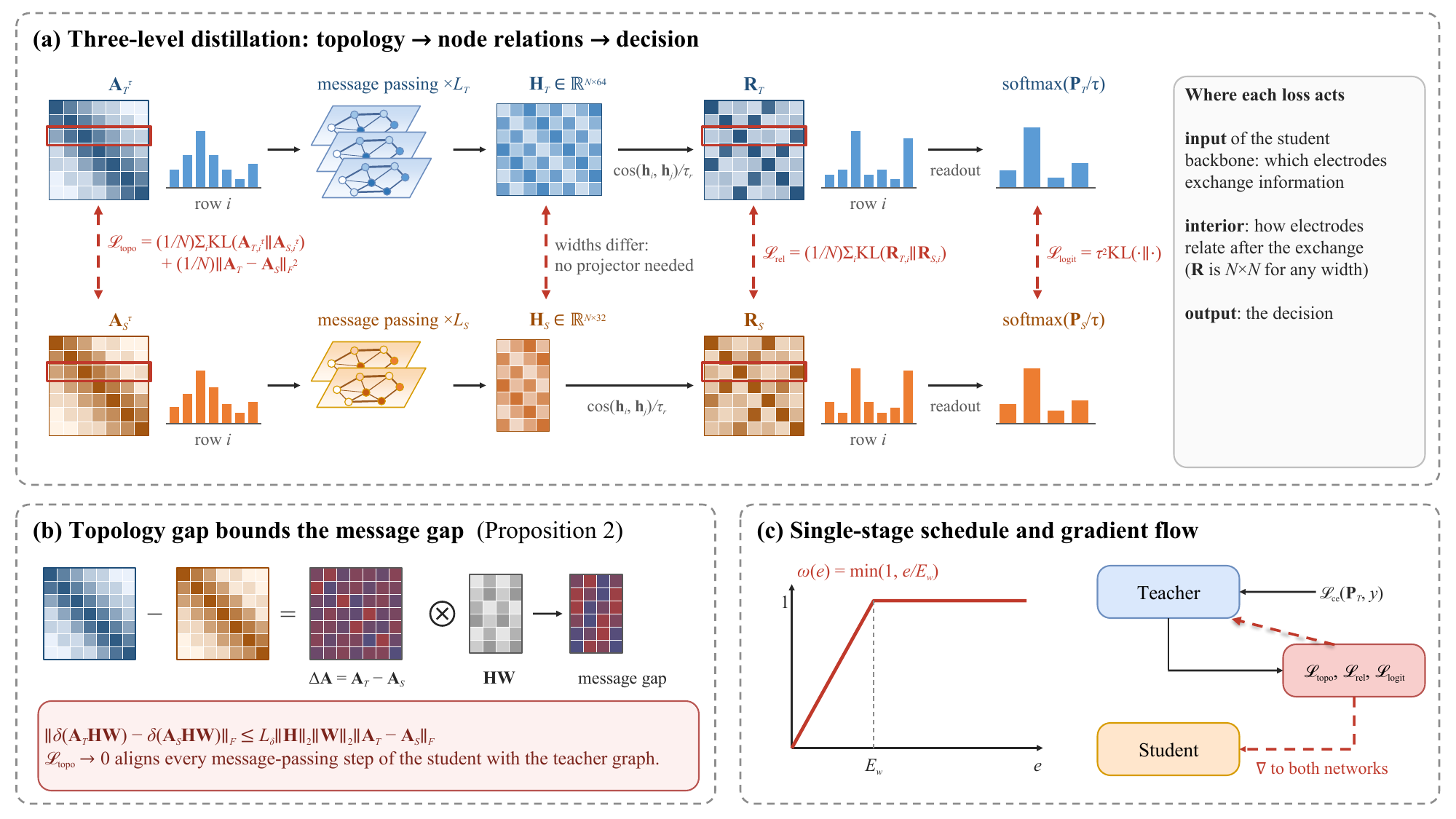}
\caption{Bidirectional multi-level topology distillation. (a)~$\mathcal{L}_{\mathrm{topo}}$ compares rows of the softened topologies, $\mathcal{L}_{\mathrm{rel}}$ rows of the $N\times N$ electrode relation matrices, which need no projection between networks of different width, and $\mathcal{L}_{\mathrm{logit}}$ the class distributions. (b)~Proposition~\ref{prop:bound}: the topology difference $\Delta\mathbf{A}$ bounds the message difference. (c)~Training: $\omega(e)$ ramps to 1 over $E_w$ epochs, and the distillation gradients reach both networks.}
\label{fig:distill}
\end{figure*}

A backbone of $L$ residual graph convolution layers~\cite{kipf2017semi} then propagates the node embeddings over the graph:
\begin{equation}
\mathbf{H}^{(l)}=\mathrm{LN}\Big(\mathbf{H}^{(l-1)}+\mathrm{Drop}\big(\mathrm{GELU}(\mathbf{A}\mathbf{H}^{(l-1)}\mathbf{W}^{(l)})\big)\Big),
\label{eq:gcn}
\end{equation}
where all layers use the adjacency matrix of the current window. The teacher uses $\mathbf{H}^{(0)}=\tilde{\mathbf{E}}$, $\mathbf{A}=\mathbf{A}_T$ and $L=L_T$, giving $\mathbf{H}_T$; mean and max pooling over the nodes give the graph representation $\mathbf{z}_g$.

For the readout, a gated multimodal unit (GMU)~\cite{arevalo2017gated} combines the graph with the eye encoding $\mathbf{z}_e=\mathrm{MLP}(\bar{\mathbf{x}}_e^{t})$ and the context $\mathbf{z}_c=\mathbf{c}_t$:
\begin{gather}
\mathbf{h}=\sum_{m\in\{g,e,c\}}\boldsymbol{\alpha}_m\odot\tanh(\mathbf{z}_m\mathbf{W}_m),\nonumber\\
[\boldsymbol{\alpha}_g;\boldsymbol{\alpha}_e;\boldsymbol{\alpha}_c]=\mathrm{softmax}_m\big([\mathbf{z}_g,\mathbf{z}_e,\mathbf{z}_c]\mathbf{W}_{\alpha}\big),
\label{eq:gmu}
\end{gather}
where the softmax runs over the three inputs for each of the $d_r$ dimensions, and a two-layer perceptron maps $\mathbf{h}$ to $\mathbf{P}_T$. Auxiliary linear heads on $\mathbf{z}_g$ and $\mathbf{z}_e$, trained with the loss $\mathcal{L}_{\mathrm{aux}}$ of Eq.~\eqref{eq:total}, keep both pathways discriminative~\cite{wang2020makes}.

The eye movements thus act at three levels: they re-weight the edges, with each row of $\mathbf{A}_T$ constrained by the physical prior; they add a gated term to each node; and they compete with the graph representation in the readout. In all three, the EEG evidence stays in place and the eye movements modulate it, in analogy to the way attention modulates cortical interaction~\cite{corbetta2002control}. The adjacency matrix $\mathbf{A}_T$ collects their effect on connectivity and serves as the interface between teacher and student.

\subsection{Self-prompted student}
\label{sec:student}

The student receives $\mathbf{X}_g^t$ only. Its generator has the same form as the teacher's, with prompts computed from EEG; the superscript $S$ marks student quantities:
\begin{gather}
\mathbf{e}^{S}_i=\mathbf{x}_{g,i}^{t}\mathbf{W}^{S}_{\mathrm{in}}+\mathbf{p}^{S}_i,\qquad
\mathbf{s}^{t}=\phi_s(\mathbf{E}^{S}),
\label{eq:sembed}\\
\boldsymbol{\pi}^{c,S}_i=\mathbf{u}^{c,S}_i\odot(\mathbf{s}^{t}\mathbf{W}^{S}_c),\qquad
\boldsymbol{\pi}^{e,S}_j=\mathbf{u}^{e,S}_j\odot(\mathbf{s}^{t}\mathbf{W}^{S}_e),
\label{eq:sprompt}\\
\mathbf{Q}^{S}=\mathbf{E}^{S}\mathbf{W}^{S}_q+\boldsymbol{\Pi}^{c,S},\qquad
\mathbf{K}^{S}=\mathbf{E}^{S}\mathbf{W}^{S}_k+\boldsymbol{\Pi}^{e,S},\nonumber\\
\mathbf{S}^{S}=\mathbf{Q}^{S}\mathbf{K}^{S\top}/\sqrt{d^{S}_k},\qquad
\mathbf{Z}_S=\mathbf{S}^{S}+\log\Aphy,\nonumber\\
\mathbf{A}_S=\rsm(\mathbf{Z}_S),
\label{eq:as}
\end{gather}
where $\phi_s(\mathbf{E}^{S})=\mathrm{vec}(\mathbf{E}^{S})\,\mathbf{W}_s+\mathbf{b}_s$. The student backbone follows Eq.~\eqref{eq:gcn} with $\mathbf{H}^{(0)}=\mathbf{E}^{S}$, $\mathbf{A}=\mathbf{A}_S$ and $L_S<L_T$ layers of width $d^{S}_h<d^{T}_h$, and a mean--max readout gives $\mathbf{P}_S$. Because both generators have the same form, $\mathbf{A}_S$ can express connectivity of the same kind as $\mathbf{A}_T$, and distillation teaches the student to approximate the teacher graph from EEG alone. The student needs neither an eye tracker nor a buffer of past windows.

\subsection{Bidirectional topology distillation and deployment}
\label{sec:distill}
As Fig.~\ref{fig:distill}(a) illustrates, three losses align the two networks at the graph, at the node representations and at the output, and their gradients reach both networks. At the graph level, with the softened topology $\mathbf{A}^{\tau}=\rsm(\mathbf{Z}/\tau)$,
\begin{equation}
\mathcal{L}_{\mathrm{topo}}=\frac{1}{N}\sum_{i=1}^{N}\KL\big(\mathbf{A}^{\tau}_{T,i}\,\big\|\,\mathbf{A}^{\tau}_{S,i}\big)+\frac{1}{N}\big\|\mathbf{A}_T-\mathbf{A}_S\big\|_F^2 .
\label{eq:ltopo}
\end{equation}
The KL term matches each electrode's distribution over all neighbors, including weak edges; the Frobenius term matches the edges that dominate message passing. Proposition~\ref{prop:bound} relates this term to the output of a graph layer; it is a standard Lipschitz bound, stated for completeness rather than as a new result.

\begin{proposition}[Topology gap bounds the message gap]
\label{prop:bound}
For node features $\mathbf{H}\in\mathbb{R}^{N\times d}$, a weight matrix $\mathbf{W}\in\mathbb{R}^{d\times d'}$ and an element-wise activation $\delta$ with Lipschitz constant $L_{\delta}$,
\begin{equation}
\begin{aligned}&\big\|\delta(\mathbf{A}_T\mathbf{H}\mathbf{W})-\delta(\mathbf{A}_S\mathbf{H}\mathbf{W})\big\|_F\\&\qquad\le L_{\delta}\,\|\mathbf{H}\|_2\,\|\mathbf{W}\|_2\,\|\mathbf{A}_T-\mathbf{A}_S\|_F .\end{aligned}
\label{eq:bound}
\end{equation}
\end{proposition}
\begin{IEEEproof}
Applied to each entry, the Lipschitz property gives $\|\delta(\mathbf{M})-\delta(\mathbf{M}')\|_F\le L_{\delta}\|\mathbf{M}-\mathbf{M}'\|_F$. Setting $\mathbf{M}-\mathbf{M}'=(\mathbf{A}_T-\mathbf{A}_S)\mathbf{H}\mathbf{W}$ and applying $\|\mathbf{X}\mathbf{Y}\|_F\le\|\mathbf{X}\|_F\|\mathbf{Y}\|_2$ and $\|\mathbf{H}\mathbf{W}\|_2\le\|\mathbf{H}\|_2\|\mathbf{W}\|_2$ gives Eq.~\eqref{eq:bound}.
\end{IEEEproof}

The bound holds for any node features and weights, in particular for those of the student. As Fig.~\ref{fig:distill}(b) illustrates, when $\mathcal{L}_{\mathrm{topo}}\to0$, propagating the student's features over $\mathbf{A}_S$ approaches propagating them over the teacher graph; layer normalization bounds $\|\mathbf{H}\|_2$.

At the node level, with $\bar{\mathbf{H}}$ the row-normalized $\mathbf{H}$,
\begin{gather}
\mathbf{R}=\rsm\big(\bar{\mathbf{H}}\bar{\mathbf{H}}^{\top}/\tau_r\big),\nonumber\\
\mathcal{L}_{\mathrm{rel}}=\frac{1}{N}\sum_{i=1}^{N}\KL\big(\mathbf{R}_{T,i}\,\big\|\,\mathbf{R}_{S,i}\big).
\label{eq:lrel}
\end{gather}
The matrix $\mathbf{R}$ is $N\times N$ in both networks, so no projection between widths is needed; unlike relational distillation across samples~\cite{park2019relational,tung2019similarity}, it relates the electrodes of one window, and $\tau_r<1$ sharpens the cosine similarities.

At the output, temperature-scaled logit distillation~\cite{hinton2015distilling} gives
\begin{equation}
\mathcal{L}_{\mathrm{logit}}=\tau^{2}\,\KL\big(\mathrm{softmax}(\mathbf{P}_T/\tau)\,\big\|\,\mathrm{softmax}(\mathbf{P}_S/\tau)\big).
\label{eq:llogit}
\end{equation}
Teacher and student are trained together with
\begin{equation}
\begin{aligned}\mathcal{L}={}&\mathcal{L}_{\mathrm{ce}}(\mathbf{P}_T,y)+\mathcal{L}_{\mathrm{ce}}(\mathbf{P}_S,y)+\lambda_a\mathcal{L}_{\mathrm{aux}}\\&+\omega(e)\big(\lambda_1\mathcal{L}_{\mathrm{topo}}+\lambda_2\mathcal{L}_{\mathrm{rel}}+\lambda_3\mathcal{L}_{\mathrm{logit}}\big),\end{aligned}
\label{eq:total}
\end{equation}
where $\mathcal{L}_{\mathrm{ce}}$ is the cross-entropy, $\mathcal{L}_{\mathrm{aux}}$ the sum of the cross-entropies of the two auxiliary heads with $\lambda_a=0.5$, $e$ the epoch and $\omega(e)=\min(1,e/E_w)$ a ramp that delays distillation while the graphs form, as Fig.~\ref{fig:distill}(c) shows. The distillation losses update both networks: the student learns the graph, relations and predictions of the teacher, and the teacher receives the gradients of the same losses. Deep mutual learning~\cite{zhang2018deep} couples peers at the output only; here the coupling also acts on the graph through which the privileged information enters, so that the two networks can be read on one connectivity. The validation protocol selects between this bidirectional form and a one-directional variant with $\sg(\cdot)$ on all teacher terms. Training is a single stage.

At deployment, systems with an eye tracker use the teacher. The tracker's validity check marks each window; windows without valid data, including those in the TCB history, enter with their eye features set to the training mean. Systems without an eye tracker can run the teacher in the same way, with all eye features at the training mean, or the student, which needs a quarter of the computation and no buffer of past windows. For comparison, Section~\ref{sec:robustness} also evaluates a dual-mode variant that uses the student for the windows without valid eye data.

\section{Experiments}
\label{sec:experiments}

\subsection{Experimental setup}
\label{sec:setup}
SEED~\cite{zheng2015investigating} and SEED-IV~\cite{zheng2019emotionmeter} provide synchronized 62-channel EEG (ESI NeuroScan) and eye movements (SMI glasses) recorded during emotional film clips. SEED has 15 subjects with three sessions of the same 15 clips, labeled positive, neutral or negative; eye movements exist for 12 subjects (1--5 and 8--14), on which all methods are evaluated. SEED-IV has 15 subjects with three sessions of 24 different clips each, labeled happy, sad, fear or neutral. The official eye-movement features cover pupil diameter, dispersion, fixations, saccades, blinks and events ($d_e=33$ on SEED, 31 on SEED-IV).

\begin{table*}[t]
\centering
\caption{Leave-one-subject-out results on SEED and SEED-IV (mean $\pm$ standard deviation over folds, \%). Worst: accuracy on the most difficult test subject. Hist.: the $K=5$ preceding windows. Best in \textbf{bold}, second best \underline{underlined}. Markers: two-sided Wilcoxon signed-rank test on per-fold accuracies with Holm correction, at the 0.05 level, EEG-only methods against the student and all other methods against the teacher; \sigup/\sigdn: significantly higher/lower than CTD-GNet.}
\label{tab:main}
\footnotesize
\setlength{\tabcolsep}{4pt}
\begin{adjustbox}{max width=\textwidth}
\begin{tabular}{@{}lllcccccc@{}}
\toprule
\multirow{2}{*}{Category} & \multirow{2}{*}{Method} & \multirow{2}{*}{Input} & \multicolumn{3}{c}{SEED (3 classes)} & \multicolumn{3}{c}{SEED-IV (4 classes)}\\
\cmidrule(lr){4-6}\cmidrule(lr){7-9}
 & & & Accuracy & Macro-F1 & Worst & Accuracy & Macro-F1 & Worst\\
\midrule
\multirow{4}{*}{EEG only} & SVM & EEG & \pmr{64.02}{6.36} & \pmr{63.76}{6.37} & 54.51 & \pmr{43.85}{4.68}\sigup & \pmr{43.46}{4.90} & 36.61\\
 & EEGNet~\cite{lawhern2018eegnet} & EEG & \pmr{51.01}{7.86}\sigdn & \pmr{47.26}{9.39} & 35.00 & \pmr{35.57}{5.33}\sigdn & \pmr{31.99}{6.16} & 26.60\\
 & DGCNN~\cite{song2018eeg} & EEG & \pmr{63.68}{6.83} & \pmr{63.34}{7.02} & 53.68 & \pmr{42.79}{5.18} & \pmr{42.47}{5.34} & 34.07\\
 & RGNN~\cite{zhong2022eeg} & EEG & \pmr{59.09}{5.42}\sigdn & \pmr{58.16}{5.70} & 50.09 & \pmr{36.86}{3.26}\sigdn & \pmr{35.85}{3.95} & 30.42\\
\midrule
\multirow{2}{*}{Eye only} & MLP & Eye & \pmr{75.08}{9.67}\sigdn & \pmr{75.10}{9.68} & 57.28 & \pmr{71.51}{10.97}\sigdn & \pmr{71.13}{11.04} & 44.70\\
 & Random forest & Eye & \pmr{74.68}{11.27}\sigdn & \pmr{74.60}{11.17} & 57.84 & \pmr{70.71}{13.64} & \pmr{70.23}{13.73} & 38.92\\
\midrule
\multirow{9}{*}{Fusion} & Feature concatenation & EEG + Eye & \pmr{82.56}{6.87}\sigdn & \pmr{82.43}{6.90} & 66.67 & \pmr{73.41}{9.44}\sigdn & \pmr{73.30}{9.31} & 50.84\\
 & BDAE~\cite{liu2016emotion,zheng2019emotionmeter} & EEG + Eye & \pmr{82.29}{7.42} & \pmr{82.21}{7.41} & 63.79 & \pmr{73.69}{9.43} & \pmr{73.56}{9.28} & 51.99\\
 & DCCA~\cite{liu2022comparing} & EEG + Eye & \pmr{76.54}{6.55}\sigdn & \pmr{76.29}{6.57} & 62.66 & \pmr{68.36}{11.76}\sigdn & \pmr{67.45}{11.79} & 36.89\\
 & DCCA-AM~\cite{liu2022comparing} & EEG + Eye & \pmr{80.87}{6.01}\sigdn & \pmr{80.69}{6.12} & \underline{68.26} & \pmr{73.10}{8.91} & \pmr{72.97}{8.84} & 50.05\\
 & TFN~\cite{zadeh2017tensor} & EEG + Eye & \pmr{80.97}{8.90}\sigdn & \pmr{80.83}{9.02} & 56.91 & \pmr{74.74}{10.00} & \pmr{74.56}{9.86} & 49.45\\
 & LMF~\cite{liu2018efficient} & EEG + Eye & \pmr{82.10}{7.43} & \pmr{81.97}{7.52} & 63.42 & \pmr{74.44}{9.89} & \pmr{74.09}{10.07} & 49.22\\
 & MISA~\cite{hazarika2020misa} & EEG + Eye & \pmr{77.78}{8.69}\sigdn & \pmr{77.61}{8.89} & 63.13 & \pmr{71.62}{8.82}\sigdn & \pmr{71.41}{8.84} & 51.94\\
 & Cross-attention~\cite{tsai2019multimodal} & EEG + Eye & \pmr{75.53}{9.42}\sigdn & \pmr{75.35}{9.52} & 57.21 & \pmr{69.69}{9.71}\sigdn & \pmr{69.40}{9.46} & 42.79\\
 & Bimodal-LSTM~\cite{tang2017multimodal} & EEG + Eye + Hist. & \pmr{\underline{83.75}}{6.60} & \pmr{\underline{83.56}}{6.79} & 67.12 & \pmr{\underline{76.05}}{9.31} & \pmr{\underline{75.86}}{9.08} & \underline{52.12}\\
\midrule
Missing- & Cross-attention + ModDrop~\cite{neverova2016moddrop} & EEG + Eye & \pmr{78.10}{8.95}\sigdn & \pmr{77.94}{9.08} & 55.66 & \pmr{70.22}{9.99}\sigdn & \pmr{69.98}{9.84} & 43.01\\
modality & MMIN~\cite{zhao2021missing} & EEG + Eye & \pmr{82.61}{7.83} & \pmr{82.49}{7.92} & 64.16 & \pmr{75.29}{9.20} & \pmr{75.13}{9.19} & 51.46\\
\midrule
\multirow{2}{*}{Ours} & CTD-GNet teacher & EEG + Eye + Hist. & \pmr{\textbf{84.91}}{5.60} & \pmr{\textbf{84.67}}{5.75} & \textbf{74.02} & \pmr{\textbf{76.17}}{7.68} & \pmr{\textbf{76.01}}{7.52} & \textbf{54.81}\\
 & CTD-GNet student & EEG & \pmr{64.10}{6.30} & \pmr{63.63}{6.30} & 53.95 & \pmr{42.04}{5.31} & \pmr{41.88}{5.43} & 34.76\\
\bottomrule
\end{tabular}
\end{adjustbox}
\end{table*}

DE features~\cite{duan2013differential} are computed from the 200-Hz EEG on 4-s non-overlapping windows in the $\delta$, $\theta$, $\alpha$, $\beta$ and $\gamma$ bands (2526 windows per subject on SEED, 2505 on SEED-IV). No smoothing across windows is applied, so on the EEG side the context bank is the only source of temporal information. The official eye-movement features, computed on the same windows and smoothed by the providers, are used identically by all methods. Because EEG and eye-movement features shift between subjects and sessions, all methods use a session-wise calibration: each feature is standardized with the mean and standard deviation of its recording session, computed over all windows of the session without labels. The calibration is transductive, since it uses the unlabeled windows of the test session, and it is applied identically to every method.

Published cross-subject accuracies on these datasets, often above 85\% for EEG with eye movements~\cite{zhu2024eeg,jimenez2024cfda}, are obtained under protocols that differ in feature smoothing, session splits, domain adaptation to the test subject and model selection. Under a strict protocol with selection on validation subjects, the benchmarks LibEER~\cite{liu2025libeer} and LibEMER~\cite{liu2025libemer} report far lower accuracies, for example 56.89\% for the best EEG--eye model on SEED in LibEMER. We therefore re-implement all baselines under one protocol instead of quoting published numbers.

All experiments use leave-one-subject-out (LOSO) cross-validation with 12 folds on SEED and 15 on SEED-IV. In each fold, two training subjects form the validation set for early stopping and all selection; the labels of the test subject are never used. For CTD-GNet, $K$, $\tau$ and the loss weights are selected by the mean validation macro-F1 over the folds in a base configuration, which uses a temporal convolutional network (TCN) as context encoder and one-directional distillation; the context encoder (GRU or TCN) and the direction of distillation are then selected in the same way. The selected final configuration uses a GRU and bidirectional distillation. Each fold is run with three seeds; we report the mean and standard deviation over folds and test differences with the two-sided Wilcoxon signed-rank test on per-fold accuracies, with Holm correction over the baselines of each comparison~\cite{holm1979simple} at the 0.05 level.

All 17 baselines are re-implemented with the same folds, features, calibration and validation subjects. The EEG-only baselines are an SVM with an RBF kernel on the 310-dimensional DE vector, EEGNet~\cite{lawhern2018eegnet} on raw 4-s segments, and DGCNN~\cite{song2018eeg} and RGNN~\cite{zhong2022eeg} without domain adaptation; the eye-only baselines are an MLP and a random forest. The fusion baselines are feature concatenation, BDAE~\cite{liu2016emotion,zheng2019emotionmeter}, DCCA~\cite{liu2022comparing,andrew2013deep}, DCCA-AM~\cite{liu2022comparing}, TFN~\cite{zadeh2017tensor}, LMF~\cite{liu2018efficient}, MISA~\cite{hazarika2020misa}, cross-modal attention~\cite{tsai2019multimodal} over 62 electrode tokens and one eye token, and Bimodal-LSTM~\cite{tang2017multimodal} with the same $K=5$ preceding windows as the teacher. The missing-modality baselines are cross-attention with ModDrop~\cite{neverova2016moddrop} at a drop rate of 0.5 and MMIN~\cite{zhao2021missing}. For EEG-only inference, the CTD-GNet student is also trained with vanilla KD~\cite{hinton2015distilling}, FitNets~\cite{romero2015fitnets}, RKD~\cite{park2019relational}, SP~\cite{tung2019similarity} and LSP~\cite{yang2020distilling}. On the validation subjects, the SVM~\cite{cortes1995support} selects $C\in\{0.1,1,10\}$ and the random forest~\cite{breiman2001random} of 500 trees its maximum depth from \{unlimited, 10, 20\}, and each distillation baseline its loss weight from $\{0.1,0.5,1,2\}$. The neural baselines follow the architectures of their original papers and share the optimizer, schedule, early stopping and validation subjects with CTD-GNet. The search of CTD-GNet covers $K$, $\tau$ and three loss weights, over which its validation accuracy varies by less than half a point, as Section~\ref{sec:ablation} shows.

The teacher has three graph layers of width $d_h=64$ with $d_k=32$, $d_s=32$ and $d_c=d_r=64$; the student has two layers of width 32 with $d_k=16$, and both use a dropout rate of 0.3. They are trained jointly with AdamW~\cite{loshchilov2019decoupled} at a learning rate of $10^{-3}$, weight decay $10^{-2}$ and batch size 128, with a cosine decay to $10^{-5}$ over at most 100 epochs, mixup~\cite{zhang2018mixup} with $\alpha=0.4$ on windows and histories, and an exponential moving average of the weights with decay 0.995~\cite{tarvainen2017mean}; each network keeps its epoch with the best validation macro-F1, with a patience of 20 epochs. The distillation ramp lasts $E_w=10$ epochs and $\tau_r=0.1$. The search covers $\lambda_i\in\{0.5,1,2\}$, $\tau\in\{2,4,8\}$ and $K\in\{3,5,7,9\}$; it selects $K=5$ on both datasets, $\tau=4$ and $(\lambda_1,\lambda_2,\lambda_3)=(2,2,2)$ on SEED, and $\tau=2$ and $(0.5,2,0.5)$ on SEED-IV.

\subsection{Comparison with existing methods}
\label{sec:main_results}
Table~\ref{tab:main} reports the results with complete eye data. The CTD-GNet teacher has the highest mean accuracy on both datasets, 84.91\% on SEED and 76.17\% on SEED-IV, and the highest macro-F1. With Holm correction for multiple comparisons, it is significantly more accurate than 9 of the 13 eye-only, fusion and missing-modality baselines on SEED and 6 of them on SEED-IV, while its margin over Bimodal-LSTM and MMIN is not significant. Its advantage is largest where the baselines are weakest: on its most difficult test subject it reaches 74.02\% on SEED and 54.81\% on SEED-IV, 5.8 and 2.7 points above the worst-subject accuracy of any baseline, and it has the smallest standard deviation across subjects among multimodal models, 5.60 and 7.68 points. Fig.~\ref{fig:subjects} shows where this consistency comes from. On SEED, the teacher has the highest lower quartile and the highest worst-subject accuracy of all models; on SEED-IV, its two most difficult test subjects reach 54.8\% and 66.6\%, against 52.1\% and 58.7\% for Bimodal-LSTM and 51.5\% and 58.5\% for MMIN. Table~\ref{tab:f1} gives the F1 score of each class: the teacher is the best model for all three emotions on SEED and for neutral and happy on SEED-IV.

\begin{figure*}[t]
\centering
\includegraphics{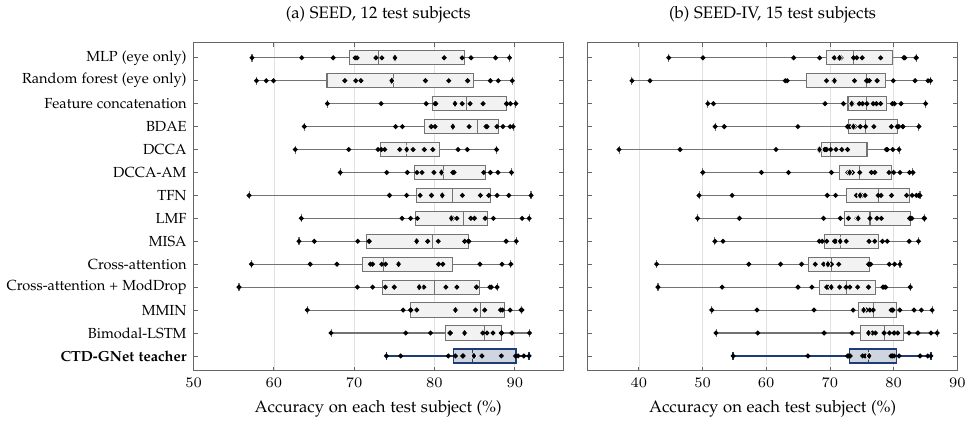}\ {}
\caption{Accuracy on each test subject for the 14 models that use eye movements. Boxes span the quartiles over the test subjects, whiskers reach the worst and the best subject, dots mark the individual subjects and diamonds the mean.}
\label{fig:subjects}
\end{figure*}

\begin{table*}[t]
\centering
\caption{F1 score of each class (\%), mean over the test subjects. Macro: mean over the classes, as in Table~\ref{tab:main}. The best value in each column is in \textbf{bold}.}
\label{tab:f1}
\footnotesize
\setlength{\tabcolsep}{5pt}
\begin{tabular}{@{}lccccccccc@{}}
\toprule
\multirow{2}{*}{Method} & \multicolumn{4}{c}{SEED} & \multicolumn{5}{c}{SEED-IV}\\
\cmidrule(lr){2-5}\cmidrule(lr){6-10}
 & Negative & Neutral & Positive & Macro & Neutral & Sad & Fear & Happy & Macro\\
\midrule
SVM & 54.8 & 57.3 & 79.3 & 63.8 & 45.9 & 43.7 & 38.7 & 45.6 & 43.5\\
EEGNet~\cite{lawhern2018eegnet} & 38.1 & 43.1 & 60.6 & 47.3 & 34.2 & 29.1 & 26.5 & 38.2 & 32.0\\
DGCNN~\cite{song2018eeg} & 53.4 & 57.5 & 79.1 & 63.3 & 43.7 & 42.0 & 38.5 & 45.6 & 42.5\\
RGNN~\cite{zhong2022eeg} & 44.5 & 51.6 & 78.3 & 58.2 & 40.3 & 33.0 & 29.5 & 40.5 & 35.9\\
\midrule
MLP (eye only) & 71.2 & 80.6 & 73.5 & 75.1 & 76.0 & 63.2 & 78.3 & 67.0 & 71.1\\
Random forest (eye only) & 71.8 & 78.8 & 73.2 & 74.6 & 75.0 & 60.8 & 78.0 & 67.0 & 70.2\\
\midrule
Feature concatenation & 78.7 & 82.5 & 86.0 & 82.4 & 77.0 & 65.2 & 78.4 & 72.6 & 73.3\\
BDAE~\cite{liu2016emotion,zheng2019emotionmeter} & 79.2 & 82.5 & 84.9 & 82.2 & 77.9 & 65.1 & 77.9 & 73.3 & 73.6\\
DCCA~\cite{liu2022comparing} & 71.8 & 76.7 & 80.3 & 76.3 & 74.5 & 57.3 & 76.4 & 61.6 & 67.5\\
DCCA-AM~\cite{liu2022comparing} & 77.1 & 80.2 & 84.7 & 80.7 & 79.1 & 63.2 & 76.9 & 72.7 & 73.0\\
TFN~\cite{zadeh2017tensor} & 77.9 & 81.9 & 82.7 & 80.8 & 78.9 & 66.9 & 78.9 & 73.6 & 74.6\\
LMF~\cite{liu2018efficient} & 79.1 & 82.5 & 84.4 & 82.0 & 78.3 & 66.0 & 78.8 & 73.2 & 74.1\\
MISA~\cite{hazarika2020misa} & 74.6 & 81.2 & 77.0 & 77.6 & 77.4 & 62.3 & 76.1 & 69.9 & 71.4\\
Cross-attention~\cite{tsai2019multimodal} & 72.8 & 78.6 & 74.7 & 75.3 & 76.6 & 60.2 & 75.1 & 65.7 & 69.4\\
Bimodal-LSTM~\cite{tang2017multimodal} & 80.3 & 83.4 & 86.9 & 83.6 & 80.8 & \textbf{68.5} & \textbf{79.3} & 74.8 & 75.9\\
\midrule
Cross-attention + ModDrop~\cite{neverova2016moddrop} & 74.8 & 79.8 & 79.3 & 77.9 & 75.8 & 61.3 & 75.2 & 67.6 & 70.0\\
MMIN~\cite{zhao2021missing} & 79.7 & 83.3 & 84.5 & 82.5 & 79.5 & 67.4 & \textbf{79.3} & 74.3 & 75.1\\
\midrule
CTD-GNet teacher & \textbf{81.2} & \textbf{84.0} & \textbf{88.8} & \textbf{84.7} & \textbf{82.3} & 67.4 & 79.2 & \textbf{75.3} & \textbf{76.0}\\
CTD-GNet student & 53.8 & 57.0 & 80.1 & 63.6 & 43.0 & 40.7 & 38.2 & 45.6 & 41.9\\
\bottomrule
\end{tabular}
\end{table*}

The strongest baseline, Bimodal-LSTM, also reads the preceding windows, which points to the value of temporal context; CTD-GNet uses this context to shape the electrode graph and is more consistent, with a worst-subject accuracy 6.9 and 2.7 points higher. The teacher exceeds the eye-only MLP by 9.8 and 4.7 points, the largest multimodal gain of all models, so the graph extracts EEG information that complements the eye movements. The student, which sees only the current EEG window, reaches 64.10\% on SEED and 42.04\% on SEED-IV, on par with DGCNN, and needs neither an eye tracker nor a window buffer.

\begin{table}[t]
\centering
\caption{Prediction without eye movements (accuracy and macro-F1, \%). Upper block: student trained with different distillation losses in the base configuration. Middle block: multimodal models tested with all eye movements removed. Lower block: CTD-GNet in the final configuration.}
\label{tab:kd}
\footnotesize
\setlength{\tabcolsep}{4pt}
\begin{tabular}{@{}lcccc@{}}
\toprule
\multirow{2}{*}{Method} & \multicolumn{2}{c}{SEED} & \multicolumn{2}{c}{SEED-IV}\\
\cmidrule(lr){2-3}\cmidrule(lr){4-5}
 & Acc & F1 & Acc & F1\\
\midrule
Student without distillation & 63.85 & 63.22 & 41.93 & 41.58\\
Vanilla KD~\cite{hinton2015distilling} & 64.46 & 63.84 & 41.85 & 41.54\\
FitNets~\cite{romero2015fitnets} & 64.00 & 63.55 & 41.99 & 41.77\\
RKD~\cite{park2019relational} & 63.92 & 63.25 & 41.55 & 41.17\\
SP~\cite{tung2019similarity} & 63.61 & 63.04 & 42.08 & 41.78\\
LSP~\cite{yang2020distilling} & 63.76 & 63.29 & 42.40 & \textbf{42.23}\\
Topology distillation (ours) & 63.96 & 63.20 & 42.49 & 42.19\\
\midrule
Cross-attention + ModDrop~\cite{neverova2016moddrop} & 62.00 & 61.02 & 38.31 & 36.41\\
MMIN~\cite{zhao2021missing} & 62.76 & 62.26 & 42.61 & 42.19\\
Bimodal-LSTM~\cite{tang2017multimodal} & 61.04 & 57.47 & 37.22 & 28.15\\
\midrule
Student without distillation & 63.66 & 63.23 & 41.98 & 41.74\\
Student, bidirectional distillation & 64.10 & 63.63 & 42.04 & 41.88\\
Teacher, eye movements removed & \textbf{66.10} & \textbf{65.33} & \textbf{43.77} & 40.59\\
\bottomrule
\end{tabular}
\end{table}

Table~\ref{tab:kd} compares predictions without eye movements; each distillation baseline is listed with the loss weight selected on the validation subjects. Topology distillation gives the most accurate distilled student on SEED-IV; its main benefits, the graph alignment and, in the bidirectional form, the gain of the teacher, are analyzed in Sections~\ref{sec:ablation} and~\ref{sec:visualization}. The student outperforms ModDrop and, unlike MMIN, needs no reconstruction network. Without its eye input, Bimodal-LSTM, the strongest baseline with complete data, falls below every student, to 61.04\% on SEED and 37.22\% on SEED-IV, whereas the CTD-GNet teacher, whose graph and context are computed from EEG, keeps 66.10\% and 43.77\%, the highest accuracy in Table~\ref{tab:kd}, although its macro-F1 on SEED-IV is below that of all students and MMIN. A device without an eye tracker can therefore run the teacher with mean eye features or, at a quarter of the computation and without a buffer of past windows, the student.

Table~\ref{tab:efficiency} compares the computational cost. The student needs 24\% of the teacher's FLOPs and processes a window in 0.21\,ms on one CPU core, three times faster than the teacher. The teacher has 59\% of the parameters of Bimodal-LSTM, 39\% of those of MMIN and 39\% of the FLOPs of cross-attention, and all models run in under 1\,ms per window, far within the 4-s update interval.

\begin{table}[t]
\centering
\caption{Computational cost per window on SEED-IV ($K=5$). FLOPs count a multiply-accumulate as two operations, with recurrent layers counted analytically; latency (ms) is measured at batch size one on one core of an Intel Xeon Gold 6430 CPU and on an NVIDIA RTX 4090 GPU. Hist.: $K$ past windows.}
\label{tab:efficiency}
\footnotesize
\setlength{\tabcolsep}{3pt}
\begin{tabular}{@{}llcccc@{}}
\toprule
Model & Input & Params & FLOPs & CPU & GPU\\
\midrule
DGCNN~\cite{song2018eeg} & EEG & 131.5k & 0.79M & 0.09 & 0.20\\
RGNN~\cite{zhong2022eeg} & EEG & 2.6k & 0.12M & 0.08 & 0.20\\
DCCA~\cite{liu2022comparing} & EEG+Eye & 120.7k & 0.24M & 0.08 & 0.15\\
TFN~\cite{zadeh2017tensor} & EEG+Eye & 204.6k & 0.41M & 0.09 & 0.16\\
LMF~\cite{liu2018efficient} & EEG+Eye & 65.3k & 0.13M & 0.10 & 0.21\\
MISA~\cite{hazarika2020misa} & EEG+Eye & 125.5k & 0.44M & 0.16 & 0.29\\
Cross-attention~\cite{tsai2019multimodal} & EEG+Eye & 174.1k & 12.43M & 0.88 & 0.92\\
MMIN~\cite{zhao2021missing} & EEG+Eye & 395.5k & 0.52M & 0.17 & 0.32\\
Bimodal-LSTM~\cite{tang2017multimodal} & EEG+Eye+Hist. & 262.7k & 3.01M & 0.36 & 0.21\\
CTD-GNet teacher & EEG+Eye+Hist. & 155.0k & 4.82M & 0.63 & 0.89\\
CTD-GNet student & EEG & 74.3k & 1.15M & 0.21 & 0.37\\
\bottomrule
\end{tabular}
\end{table}

A recognizer in daily use meets new people and new stimuli at the same time. Because every subject of SEED and SEED-IV watches the same film clips~\cite{kilgallen2025repeated}, and the pupil responds to the luminance of a clip as well as to arousal~\cite{bradley2008pupil}, we also evaluate a stimulus-disjoint protocol that tests both kinds of generalization together. In each fold, one clip per emotion, on SEED-IV one per emotion and session, is removed from all training and validation subjects, and the test subject is evaluated on these clips only; the held-out clips rotate over the folds. Table~\ref{tab:stimulus} reports the accuracy on these unseen clips. On SEED, the CTD-GNet teacher is the most accurate model, 77.1\%, and its lead widens under this protocol: over Bimodal-LSTM from 1.2 points in Table~\ref{tab:main} to 4.5 points, and over MMIN from 2.3 to 5.9 points. With Holm correction, it is significantly more accurate than the eye-only MLP and DGCNN, $p=0.017$ for both, whereas its lead over the fusion models is not significant, $p\geq0.19$. On SEED-IV, it matches the best model, Bimodal-LSTM, with 62.1\% against 62.3\%, and is 3.2 points above MMIN. All fusion models stay far above EEG alone, by about 20 points on SEED-IV, so the eye movements carry emotion information that transfers to new stimuli.

\begin{table}[t]
\centering
\caption{Generalization to new subjects and new stimuli: accuracy (\%) on film clips that no training or validation subject has seen. In each fold, one clip per emotion, on SEED-IV one per emotion and session, is held out from all training and validation subjects.}
\label{tab:stimulus}
\footnotesize
\setlength{\tabcolsep}{6pt}
\begin{tabular}{@{}lcc@{}}
\toprule
Method & SEED & SEED-IV\\
\midrule
DGCNN~\cite{song2018eeg} & 63.4 & 39.8\\
MLP (eye only) & 63.2 & 57.6\\
\midrule
Feature concatenation & 70.2 & 61.7\\
Bimodal-LSTM~\cite{tang2017multimodal} & 72.6 & 62.3\\
MMIN~\cite{zhao2021missing} & 71.2 & 58.9\\
\midrule
CTD-GNet student & 64.2 & 38.9\\
CTD-GNet teacher & \textbf{77.1} & 62.1\\
\bottomrule
\end{tabular}
\end{table}

\begin{table}[t]
\centering
\caption{Ablation of the teacher (accuracy and macro-F1, \%). Upper block: the final configuration and variants of it. Lower block: components removed from the base configuration, which uses a TCN context encoder and one-directional distillation. $^{\dagger}$:~$p<0.05$ against the first row of the block, two-sided Wilcoxon signed-rank test on per-fold accuracies, uncorrected; with Holm correction within each block, no difference reaches the 0.05 level.}
\label{tab:ablation_teacher}
\footnotesize
\setlength{\tabcolsep}{3.5pt}
\begin{tabular}{@{}lcccc@{}}
\toprule
\multirow{2}{*}{Variant} & \multicolumn{2}{c}{SEED} & \multicolumn{2}{c}{SEED-IV}\\
\cmidrule(lr){2-3}\cmidrule(lr){4-5}
 & Acc & F1 & Acc & F1\\
\midrule
Final (GRU, bidirectional) & 84.91 & 84.67 & 76.17 & 76.01\\
\quad TCN instead of GRU & 83.92$^{\dagger}$ & 83.68 & 75.65 & 75.43\\
\quad One-directional distillation & 83.71$^{\dagger}$ & 83.52 & 75.90 & 75.74\\
\quad No distillation & 83.71$^{\dagger}$ & 83.53 & 75.82 & 75.65\\
\quad Bidirectional, logit loss only & 84.57 & 84.35 & 76.25 & 76.08\\
\quad Bidirectional, topology loss only & 83.74 & 83.55 & 75.98 & 75.79\\
\quad Eye movements only in the graph & 84.81 & 84.60 & 75.97 & 75.82\\
\quad Eye movements only in the readout & 84.45 & 84.27 & 76.40 & 76.19\\
\midrule
Base (TCN, one-directional) & 82.19 & 82.00 & 75.71 & 75.53\\
\quad w/o temporal context & 81.09 & 80.98 & 73.78$^{\dagger}$ & 73.50\\
\quad w/o both prompts & 81.74 & 81.60 & 73.75 & 73.52\\
\quad Static learnable adjacency & 80.45$^{\dagger}$ & 80.28 & 74.32 & 74.09\\
\quad w/o mixup & 82.12 & 82.03 & 74.90 & 74.60\\
\quad w/o weight averaging & 81.87 & 81.68 & 74.98 & 74.82\\
\bottomrule
\end{tabular}
\end{table}

\subsection{Ablation study}
\label{sec:ablation}
The upper part of Table~\ref{tab:ablation_teacher} varies the final configuration. Replacing the GRU by a TCN, making the distillation one-directional or removing it lowers the teacher on both datasets, by 1.0 to 1.2 points on SEED, with uncorrected $p$ between 0.027 and 0.044, and by 0.3 to 0.5 points on SEED-IV. One-directional distillation matches no distillation, because the stop-gradient isolates the teacher. Coupling only the predictions in both directions gives 84.57\% and 76.25\%, close to the full objective, whereas coupling only the graphs leaves the teacher at the level without distillation: the gain of the teacher is consistent with mutual learning through the predictions~\cite{zhang2018deep}, and the topology coupling provides the 35-fold graph alignment of Section~\ref{sec:visualization}, which is its direct target. The eye movements can reach the teacher through either route. With the eye movements only in the graph, through the edge prompts, the node pathway and the context bank, or only in the readout, the teacher stays within 0.5 points of the full model on both datasets. Modulating the EEG graph thus carries the information of the eye movements without any eye features in the readout, which is the premise of CTD-GNet.

The lower part removes components from the base configuration, and every removal lowers the accuracy on both datasets. The largest effects come from the components that make the graph change with every window: without the temporal context the teacher loses 1.1 and 1.9 points, a static adjacency shared by all windows, as in DGCNN, costs 1.7 and 1.4 points, and removing both prompts costs 0.5 and 2.0 points, so the graph should follow both the current window and its history. Mixup and weight averaging add 0.1 to 0.8 points. Fig.~\ref{fig:position} shows that on SEED the accuracy rises during the first minutes of a trial, for the teacher as for the student, as the induced emotion builds up.

\begin{figure}[t]
\centering
\includegraphics{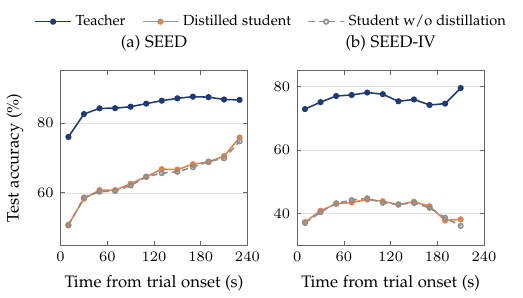}\ {}%
\caption{Test accuracy versus the position of the window within the trial, in 20-s bins plotted at the bin centre, over all folds and seeds. Bins with fewer than 1000 test windows are omitted.}
\label{fig:position}
\end{figure}

The hyperparameters have small effects, as Fig.~\ref{fig:sensitivity} shows: over histories of three to nine windows, over $\tau$ and over all 27 combinations of the loss weights, the validation accuracy varies by less than half a point, so CTD-GNet needs no careful tuning.

\begin{figure*}[t]
\centering
\includegraphics{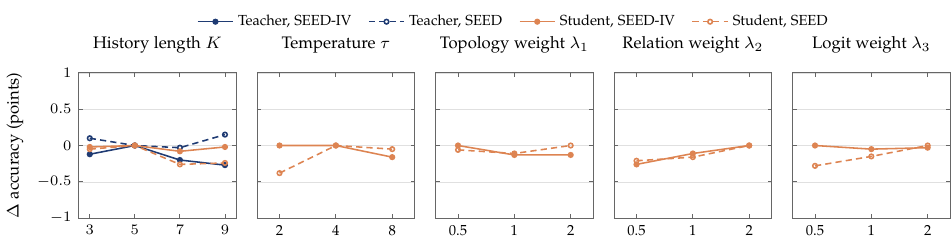}\ {}
\caption{Validation accuracy over the hyperparameter grid, relative to the selected configuration, which lies at zero. Each panel varies one hyperparameter and keeps the others at their selected values. The temperature and the loss weights are searched in the base configuration, whose stop-gradient leaves the teacher unaffected by them, so only the student is shown in these panels.}
\label{fig:sensitivity}
\end{figure*}

\begin{figure*}[t]
\centering
\includegraphics{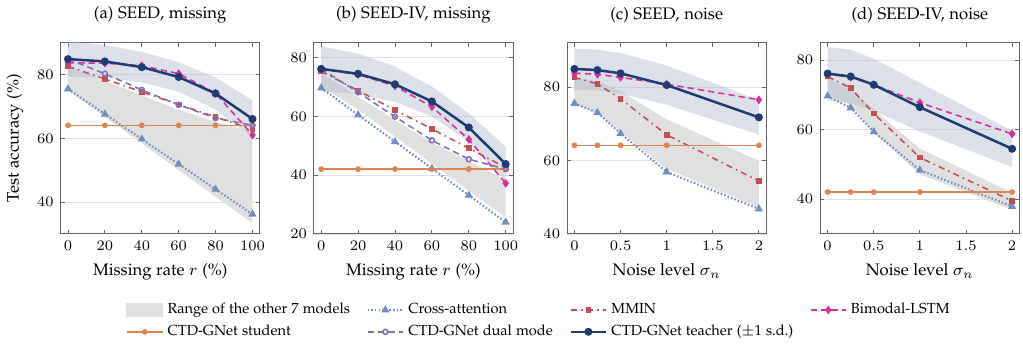}\ {}
\caption{Accuracy under missing (a, b) and noisy (c, d) eye-movement data, mean over the test subjects. Blue band: teacher $\pm$ one standard deviation over the test subjects. Grey band: range of the other seven models that use eye movements, namely the eye-only MLP, feature concatenation, DCCA-AM, TFN, LMF, MISA and ModDrop. The dual mode uses the student for windows without eye data; under noise, it differs from the teacher only on the few windows whose eye data are invalid in the recordings.}
\label{fig:robust}
\end{figure*}

\begin{table*}[t]
\centering
\caption{Accuracy (\%) under missing and noisy eye-movement data, mean over the test subjects and five random draws. The results without corruption are those of Table~\ref{tab:main}. The best value in each column is in \textbf{bold}.}
\label{tab:robust}
\footnotesize
\setlength{\tabcolsep}{5pt}
\begin{tabular}{@{}llccccccccc@{}}
\toprule
\multirow{2}{*}{Dataset} & \multirow{2}{*}{Method} & \multicolumn{5}{c}{Missing rate $r$} & \multicolumn{4}{c}{Noise level $\sigma_n$}\\
\cmidrule(lr){3-7}\cmidrule(lr){8-11}
 & & 20\% & 40\% & 60\% & 80\% & 100\% & 0.25 & 0.5 & 1 & 2\\
\midrule
\multirow{13}{*}{SEED} & MLP (eye only) & 66.7 & 58.5 & 50.0 & 41.7 & 33.4 & 73.0 & 67.8 & 57.7 & 47.2\\
\cmidrule(l){2-11}
 & Feature concatenation & 78.1 & 73.5 & 69.0 & 64.4 & 60.0 & 81.6 & 78.9 & 71.2 & 58.5\\
 & DCCA-AM~\cite{liu2022comparing} & 76.4 & 71.8 & 67.3 & 62.8 & 58.3 & 79.6 & 76.5 & 69.4 & 59.9\\
 & TFN~\cite{zadeh2017tensor} & 75.9 & 70.7 & 65.5 & 60.4 & 55.3 & 79.6 & 75.8 & 67.3 & 56.6\\
 & LMF~\cite{liu2018efficient} & 76.9 & 71.6 & 66.3 & 61.1 & 55.9 & 80.1 & 76.0 & 67.2 & 56.5\\
 & MISA~\cite{hazarika2020misa} & 71.1 & 64.5 & 57.7 & 51.1 & 44.5 & 74.9 & 69.3 & 59.5 & 49.8\\
 & Cross-attention~\cite{tsai2019multimodal} & 67.6 & 59.8 & 51.9 & 44.0 & 36.2 & 73.0 & 67.3 & 56.9 & 46.8\\
 & Bimodal-LSTM~\cite{tang2017multimodal} & 83.6 & \textbf{82.8} & \textbf{80.3} & \textbf{74.2} & 61.0 & 83.4 & 82.6 & \textbf{80.8} & \textbf{76.5}\\
\cmidrule(l){2-11}
 & Cross-attention + ModDrop~\cite{neverova2016moddrop} & 74.9 & 71.6 & 68.4 & 65.2 & 62.0 & 75.7 & 70.5 & 59.7 & 48.1\\
 & MMIN~\cite{zhao2021missing} & 78.7 & 74.6 & 70.7 & 66.7 & 62.8 & 80.8 & 76.8 & 67.0 & 54.4\\
\cmidrule(l){2-11}
 & CTD-GNet student & 64.1 & 64.1 & 64.1 & 64.1 & 64.1 & 64.1 & 64.1 & 64.1 & 64.1\\
 & CTD-GNet dual mode & 80.3 & 75.2 & 70.5 & 66.4 & 64.1 & \textbf{84.6} & \textbf{83.7} & 80.5 & 71.7\\
 & CTD-GNet teacher & \textbf{84.2} & 82.4 & 79.3 & 74.1 & \textbf{66.1} & \textbf{84.6} & \textbf{83.7} & 80.5 & 71.7\\
\midrule
\multirow{13}{*}{SEED-IV} & MLP (eye only) & 62.3 & 53.1 & 44.0 & 34.8 & 25.7 & 68.6 & 61.5 & 49.1 & 37.7\\
\cmidrule(l){2-11}
 & Feature concatenation & 65.7 & 57.9 & 50.2 & 42.5 & 34.8 & 71.3 & 65.8 & 54.6 & 41.9\\
 & DCCA-AM~\cite{liu2022comparing} & 65.0 & 57.0 & 48.9 & 40.8 & 32.8 & 70.2 & 63.6 & 52.1 & 40.9\\
 & TFN~\cite{zadeh2017tensor} & 66.7 & 58.7 & 50.7 & 42.7 & 34.7 & 71.7 & 65.1 & 52.8 & 40.6\\
 & LMF~\cite{liu2018efficient} & 66.4 & 58.4 & 50.3 & 42.3 & 34.3 & 71.2 & 64.1 & 51.6 & 39.8\\
 & MISA~\cite{hazarika2020misa} & 62.4 & 53.2 & 44.0 & 34.7 & 25.5 & 67.3 & 59.2 & 47.2 & 36.9\\
 & Cross-attention~\cite{tsai2019multimodal} & 60.5 & 51.4 & 42.3 & 33.1 & 24.0 & 66.3 & 59.4 & 48.3 & 37.9\\
 & Bimodal-LSTM~\cite{tang2017multimodal} & 74.3 & 70.4 & 63.2 & 52.1 & 37.2 & 75.1 & \textbf{73.0} & \textbf{67.8} & \textbf{58.8}\\
\cmidrule(l){2-11}
 & Cross-attention + ModDrop~\cite{neverova2016moddrop} & 63.8 & 57.4 & 51.1 & 44.7 & 38.3 & 67.1 & 60.3 & 49.0 & 38.3\\
 & MMIN~\cite{zhao2021missing} & 68.7 & 62.2 & 55.7 & 49.1 & 42.6 & 72.0 & 64.7 & 52.0 & 39.5\\
\cmidrule(l){2-11}
 & CTD-GNet student & 42.0 & 42.0 & 42.0 & 42.0 & 42.0 & 42.0 & 42.0 & 42.0 & 42.0\\
 & CTD-GNet dual mode & 68.2 & 59.9 & 51.8 & 45.4 & 42.0 & \textbf{75.3} & \textbf{73.0} & 66.5 & 54.5\\
 & CTD-GNet teacher & \textbf{74.5} & \textbf{70.9} & \textbf{65.1} & \textbf{56.2} & \textbf{43.8} & \textbf{75.3} & 72.9 & 66.5 & 54.5\\
\bottomrule
\end{tabular}
\end{table*}

\subsection{Robustness to missing and noisy eye data}
\label{sec:robustness}
At test time, a fraction $r\in\{0,20,\dots,100\}\%$ of the windows lose their eye features, which are replaced by the training mean, also in the TCB history; in a second test, Gaussian noise with $\sigma_n\in\{0,0.25,0.5,1,2\}$ in standardized units is added to all eye features. Each setting is repeated with five random draws. Fig.~\ref{fig:robust} plots the accuracy, and Table~\ref{tab:robust} lists the values of the main models.

With missing eye data, the teacher is the most accurate model at every missing rate on SEED-IV, and at 80\% it exceeds Bimodal-LSTM by 4.1 points and MMIN by 7.1 points, as Fig.~\ref{fig:robust}(a)--(b) shows. On SEED, it is the most accurate model with complete eye data, at 20\% and when all eye data are missing. Models that read only the current window degrade almost linearly; without eye data, cross-attention falls to 36.2\% on SEED and 24.0\% on SEED-IV. The teacher degrades gracefully, likely because its context and eye cue still hold the valid eye data of preceding windows and its graph always carries the EEG evidence; without any eye data it keeps 66.1\% on SEED, more than any EEG-only model, and 43.8\% on SEED-IV. Whenever eye data are missing, the teacher alone is more accurate than the dual-mode variant, so the full system needs no switching mechanism.

With noisy eye data, the teacher is the most accurate model up to $\sigma_n=0.5$ on SEED and at $\sigma_n=0.25$ on SEED-IV, and second only to Bimodal-LSTM at the higher noise levels. At $\sigma_n=1$, it keeps 80.5\% on SEED and 66.5\% on SEED-IV, 9.3 and 11.9 points above every single-window fusion and missing-modality model in Table~\ref{tab:robust}. Two design choices likely contribute: averaging the eye cue over windows attenuates independent noise, and in the edge and node routes corrupted eye features can only re-weight the edges and gate the node terms of an EEG graph, with the row-wise softmax and the log-prior bounding their influence on the edges.

\subsection{Analysis of the learned graphs}
\label{sec:visualization}

\begin{figure}[!t]
\centering
\includegraphics[width=\columnwidth,page=3]{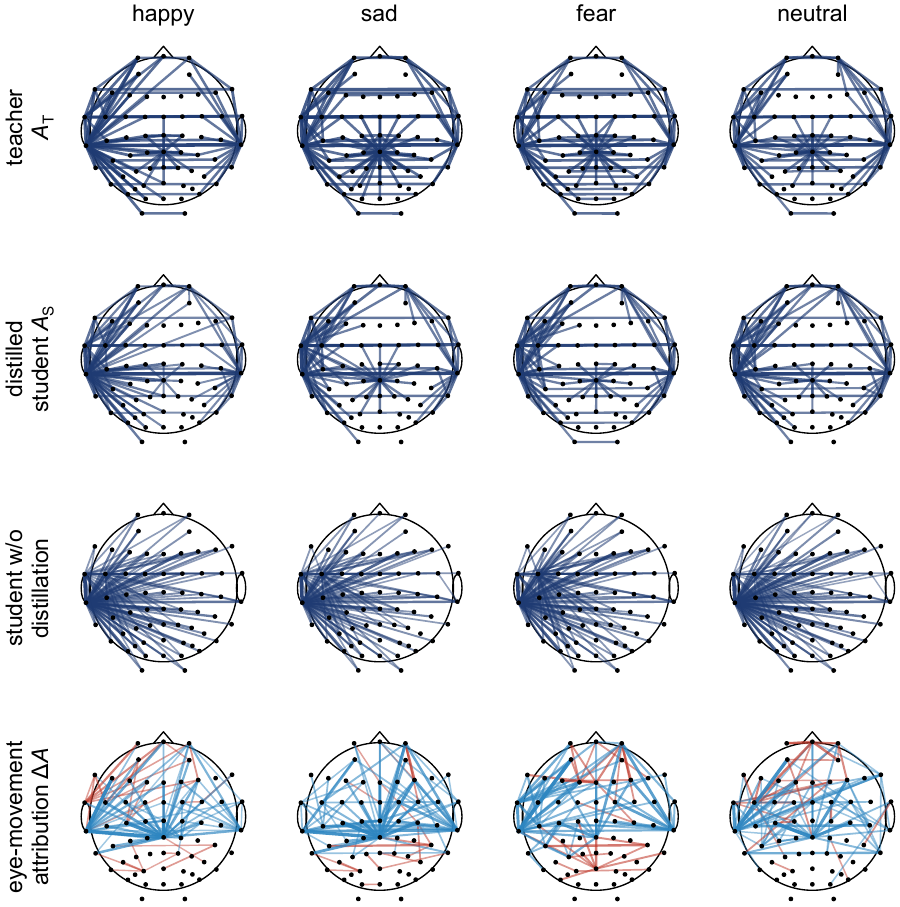}
\setlength{\normallineskiplimit}{1pt}\setlength{\normallineskip}{1.5pt}\setlength{\lineskiplimit}{1pt}\setlength{\lineskip}{1.5pt}%
\caption{Learned graphs per emotion on SEED-IV. Rows: teacher topology $\mathbf{A}_T$, distilled student topology $\mathbf{A}_S$, student topology without distillation, and eye-movement attribution $\Delta\mathbf{A}=\mathbb{E}_t[\mathbf{A}_T-\mathbf{A}_T^{(\mathbf{x}_e=\mathbf{0})}]$, where the eye features of all windows are set to their calibrated mean of zero; red edges gain and blue edges lose weight through the eye movements. Each panel shows the 5\% strongest edges averaged over the test windows of the class.}
\label{fig:topology}
\end{figure}

\begin{figure}[!t]
\centering
\includegraphics[width=0.75\columnwidth,page=4]{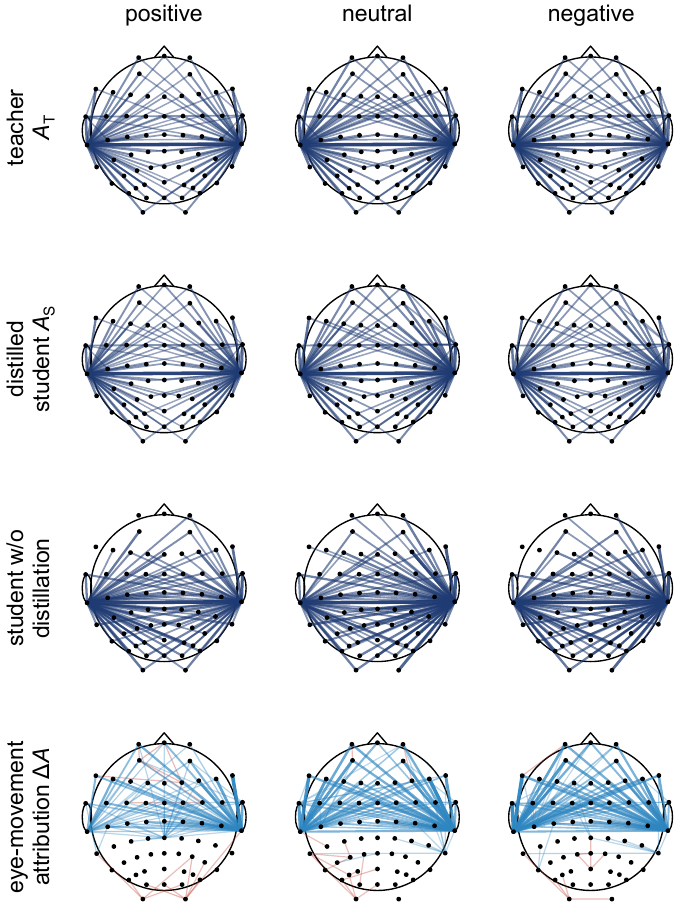}
\caption{Learned graphs per emotion on SEED, drawn as in Fig.~\ref{fig:topology}.}
\label{fig:topology_seed}
\end{figure}

\begin{figure*}[t]
\centering
\includegraphics{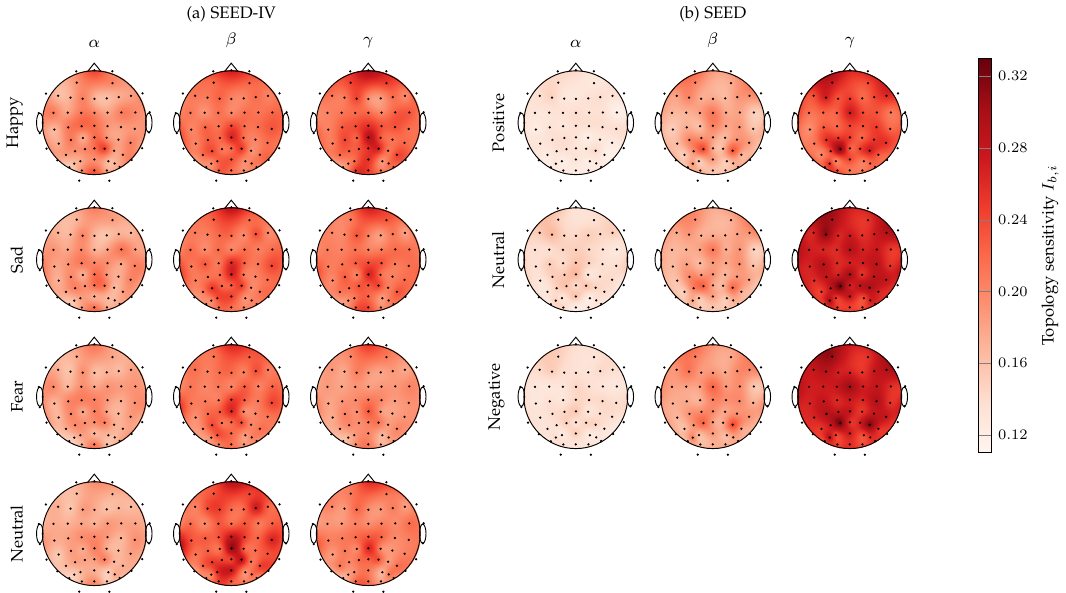}\ {}
\caption{Band-wise topology sensitivity $I_{b,i}$ of the teacher, defined in Eq.~\eqref{eq:bandsens}, for the $\alpha$, $\beta$ and $\gamma$ bands and each emotion, averaged over the test windows of all folds and seeds. The values at the 62 electrodes (dots) are interpolated with a thin-plate spline; nose up, left hemisphere on the left, one colour scale for both datasets.}
\label{fig:bands}
\end{figure*}

\begin{table}[t]
\centering
\caption{Agreement between the teacher graph $\mathbf{A}_T$ and the student graph $\mathbf{A}_S$ on the test windows: mean row-wise KL divergence $\KL(\mathbf{A}_{T,i}\|\mathbf{A}_{S,i})$ and Jaccard overlap of the 10\% strongest edges, mean $\pm$ standard deviation over the 36 (SEED) and 45 (SEED-IV) runs of all folds and seeds.}
\label{tab:align}
\footnotesize
\setlength{\tabcolsep}{3.5pt}
\begin{tabular}{@{}lcccc@{}}
\toprule
\multirow{2}{*}{Distillation} & \multicolumn{2}{c}{SEED} & \multicolumn{2}{c}{SEED-IV}\\
\cmidrule(lr){2-3}\cmidrule(lr){4-5}
 & KL $\downarrow$ & Jaccard $\uparrow$ & KL $\downarrow$ & Jaccard $\uparrow$\\
\midrule
None & \pmr{7.55}{1.42} & \pmr{0.209}{0.047} & \pmr{7.90}{1.47} & \pmr{0.093}{0.030}\\
Bidirectional &\pmr{\textbf{0.21}}{0.09} & \pmr{\textbf{0.414}}{0.131} & \pmr{\textbf{0.23}}{0.06} & \pmr{\textbf{0.226}}{0.063}\\
\bottomrule
\end{tabular}
\end{table}

Bidirectional distillation brings the teacher and student graphs close together, as Table~\ref{tab:align} shows. Compared with a student trained without distillation, the mean row-wise KL divergence between $\mathbf{A}_T$ and $\mathbf{A}_S$ falls from 7.90 to 0.23 on SEED-IV and from 7.55 to 0.21 on SEED, and the Jaccard overlap of the 10\% strongest edges rises from 0.093 to 0.226 and from 0.209 to 0.414. From a single EEG window, the distilled student shares 2.0 to 2.4 times as many of the teacher's strongest edges as the student without distillation, so both networks can be read on one connectivity.

On SEED-IV, the teacher graph of each emotion is organized around bilateral hubs at the lateral temporal electrodes, as Fig.~\ref{fig:topology} shows, and for sadness and fear a centro-parietal midline hub collects edges from across the scalp. The distilled student shares these hubs, whereas the student without distillation concentrates its edges on one left temporal electrode for every emotion. On SEED, Fig.~\ref{fig:topology_seed} shows that the teacher and the distilled student share bilateral hubs at the lateral temporal electrodes, the region that Zheng and Lu~\cite{zheng2015investigating} found most discriminative.

In the eye-movement attribution $\Delta\mathbf{A}$ of Figs.~\ref{fig:topology} and~\ref{fig:topology_seed}, the edges that lose weight link the temporal hubs to frontal and central electrodes. Table~\ref{tab:regions} quantifies this reorganization by the attention that the electrodes of each scalp region receive, $\frac{100}{N}\sum_{i}\sum_{j\in\mathcal{R}}\Delta A_{ij}$ for a region $\mathcal{R}$; because every row of $\mathbf{A}_T$ sums to one, the changes sum to zero over the regions. On both datasets and for every emotion, the eye movements take attention away from the temporal electrodes, by 1.30 to 2.43 points on SEED and 0.24 to 0.55 points on SEED-IV, and give it to the frontal electrodes, by 0.49 to 0.95 and 0.27 to 0.58 points; the occipital electrodes gain in six of the seven emotion classes. This shift toward frontal electrodes, near the eyes and the frontal eye fields of the dorsal attention network~\cite{corbetta2002control}, and toward occipital electrodes over the visual cortex is consistent with the involvement of these regions in eye movements and visual processing.

\begin{table}[t]
\centering
\setlength{\normallineskiplimit}{1pt}\setlength{\normallineskip}{1.5pt}\setlength{\lineskiplimit}{1pt}\setlength{\lineskip}{1.5pt}%
\caption{Change of the attention received by each scalp region through the eye movements, $\frac{100}{N}\sum_i\sum_{j\in\mathcal{R}}\Delta A_{ij}$ (points), for the teacher graph averaged over the test windows of each emotion. Frontal: FP, AF and F electrodes (14); central: FC and C (14); temporal: FT7, FT8, T7, T8, TP7 and TP8 (6); parietal: CP and P (16); occipital: PO, O and CB (12). Each row sums to zero.}
\label{tab:regions}
\footnotesize
\setlength{\tabcolsep}{3.4pt}
\begin{tabular}{@{}llrrrrr@{}}
\toprule
Dataset & Emotion & Frontal & Central & Temporal & Parietal & Occipital\\
\midrule
\multirow{3}{*}{SEED} & Negative & \phantom{$-$}0.70 & \phantom{$-$}0.33 & $-$2.05 & \phantom{$-$}0.48 & \phantom{$-$}0.54\\
 & Neutral & \phantom{$-$}0.95 & \phantom{$-$}0.14 & $-$2.43 & \phantom{$-$}0.54 & \phantom{$-$}0.80\\
 & Positive & \phantom{$-$}0.49 & \phantom{$-$}0.14 & $-$1.30 & \phantom{$-$}0.16 & \phantom{$-$}0.51\\
\midrule
\multirow{4}{*}{SEED-IV} & Fear & \phantom{$-$}0.38 & $-$0.01 & $-$0.55 & $-$0.31 & \phantom{$-$}0.50\\
 & Happy & \phantom{$-$}0.34 & \phantom{$-$}0.16 & $-$0.24 & $-$0.52 & \phantom{$-$}0.25\\
 & Neutral & \phantom{$-$}0.58 & \phantom{$-$}0.00 & $-$0.27 & $-$0.27 & $-$0.04\\
 & Sad & \phantom{$-$}0.27 & $-$0.16 & $-$0.31 & $-$0.14 & \phantom{$-$}0.34\\
\bottomrule
\end{tabular}
\end{table}

To relate the graph to the frequency bands, we define the band-wise topology sensitivity of electrode $i$ as
\begin{equation}
I_{b,i}=\mathbb{E}_t\Big[\sum_{j=1}^{N}\big|A_{T,ij}-A^{(\setminus b)}_{T,ij}\big|\Big],
\label{eq:bandsens}
\end{equation}
where $\mathbf{A}^{(\setminus b)}_T$ is computed with band $b$ set to its calibrated mean. Fig.~\ref{fig:bands} maps $I_{b,i}$ over the scalp for the $\alpha$, $\beta$ and $\gamma$ bands. On SEED-IV, for every emotion, the mean sensitivity over the electrodes is higher in the $\beta$ and $\gamma$ bands, 0.22 to 0.23 and 0.19 to 0.23, than in the $\alpha$ band, 0.18 to 0.19, and in each of these bands the most sensitive electrode is the midline electrode CPz or FPz or the left temporal electrode T7. On SEED, the sensitivity grows from the $\alpha$ band, 0.13 to 0.14, to the $\beta$ band, 0.18 to 0.19, and the $\gamma$ band, 0.24 to 0.27, and the parietal electrode P1 is the most sensitive in the $\gamma$ band for every emotion. This emphasis is consistent with Zheng and Lu~\cite{zheng2015investigating}, who found the $\beta$ and $\gamma$ bands most discriminative.

\section{Discussion}
\label{sec:discussion}

\subsection{Eye movements as modulators of cortical connectivity}
\label{sec:disc_edges}
CTD-GNet tests the hypothesis that the emotion information in eye movements and temporal context can be expressed as a modulation of how cortical regions interact~\cite{corbetta2002control}. The results are consistent with it. The components that make the graph depend on the window and its context have the largest ablation effects, and the teacher keeps its accuracy when the eye movements only modulate the EEG graph and never enter the readout, so modulating the graph is a sufficient route for the eye information. The eye movements shift the attention of the graph from the temporal electrodes toward frontal electrodes for every emotion and toward occipital electrodes in six of seven classes, which may reflect the role of these regions in eye movements and vision. The model that keeps EEG as the substrate is the most consistent across subjects, has the highest mean accuracy, degrades gracefully without eye data and is the most accurate on film clips unseen in training on SEED. Eye movements are the stronger single modality on both datasets, yet this model gains the most over eye-only recognition: the role of a signal in a model need not follow its standalone accuracy.

\subsection{What bidirectional distillation provides}
Learning using privileged information usually passes outputs or features from teacher to student~\cite{lopezpaz2016unifying}. CTD-GNet couples the two networks in both directions, at the predictions, the node relations and the adjacency matrix, and the coupling serves two purposes. It improves the teacher: one-directional distillation leaves the teacher at the level without distillation, whereas the bidirectional form raises it by 1.2 points on SEED, and coupling the predictions alone gives most of this gain, in line with deep mutual learning~\cite{zhang2018deep}, where networks trained together improve each other. The coupling of the graphs reduces the teacher--student graph divergence 35-fold, so that the EEG-only student, which runs without an eye tracker, approximates the graph of the multimodal teacher and both networks can be read on one graph. Coupling with an EEG-only student can thus improve the multimodal model itself, not only the deployable one.

\subsection{Implications}
Emotion-aware systems for mental health, education, driving and virtual reality must work for new users and with imperfect sensors, yet the fusion baselines fall below EEG-only accuracy when the eye tracker fails. Consistency across subjects and graceful degradation, the two strengths of CTD-GNet, address exactly these conditions. The principle extends to any auxiliary signal that describes how the parts of a primary signal interact, such as facial video or peripheral physiology for EEG: Proposition~\ref{prop:cancel} holds for every attention graph with additive cues, and bidirectional topology distillation aligns the graph of the multimodal model with that of a model of the primary signal alone.

\subsection{Limitations and future work}
The evaluation covers two film-clip datasets from the same laboratory, with the officially smoothed eye-movement features and a session-wise calibration that uses the unlabeled windows of the test session, which an online system would replace by a calibration recording or running statistics. On film clips unseen in training, all models that use eye movements lose 10 to 14 points, against 1.6 to 3.2 points for EEG alone, so part of the eye-movement accuracy is specific to the clips. The margins over Bimodal-LSTM and MMIN are not significant, the student gains little accuracy from distillation, and the frontal gain in Table~\ref{tab:regions} may partly reflect electro-oculographic (EOG) activity, for which we ran no artifact-removal control. Future work will extend CTD-GNet to further datasets with more stimuli per emotion, to recordings outside the laboratory, to raw gaze features and to EOG-corrected EEG, and will give the student a longer EEG context.

\section{Conclusion}
\label{sec:conclusion}

We presented CTD-GNet, in which EEG forms the nodes of an electrode graph and eye movements and temporal context modulate its edges, nodes and readout. Coupling the multimodal teacher with a single-window EEG student in both directions improves the teacher and aligns the two graphs 35-fold, so that the EEG-only student approximates the teacher graph. Against 17 baselines on SEED and SEED-IV, CTD-GNet achieves the highest mean accuracy, 84.91\% and 76.17\%, with the smallest spread across subjects among multimodal models and the highest worst-subject accuracy, 74.02\% and 54.81\%. It is the most accurate model at every rate of missing eye data on SEED-IV, keeps 66.1\% on SEED without eye data, more than any EEG-only model, and is the most accurate on film clips unseen in training on SEED. Assigning each signal its physiological role thus yields multimodal recognizers that stay consistent across users and robust to sensor failure.

\bibliographystyle{IEEEtran}
\bibliography{cas-refs}

\end{document}